\documentclass{article} %
\usepackage{iclr2023_conference,times}
\iclrfinalcopy

\usepackage{hyperref}
\usepackage{url}
\usepackage{graphicx}

\usepackage[utf8]{inputenc} %
\usepackage[T1]{fontenc}    %

\usepackage{booktabs}       %
\usepackage{amsfonts}       %
\usepackage{nicefrac}       %
\usepackage{microtype}      %
\usepackage{lipsum}
\usepackage{xcolor}         %
\usepackage{todonotes}
\usepackage{sidecap}
\usepackage[export]{adjustbox}
\usepackage{wrapfig}

\usepackage{algorithm}
\usepackage{algorithmic}
\usepackage{enumitem}
\setitemize{noitemsep,topsep=0pt,parsep=0pt,partopsep=0pt,leftmargin=12pt}

\usepackage{amsmath, amssymb, amsthm,bm}
\usepackage{mathtools}
\usepackage{cancel}
\usepackage{mathtools}
\usepackage{empheq}
\usepackage{caption}
\usepackage{graphicx}
\usepackage{booktabs}
\usepackage{nccmath}

\DeclarePairedDelimiterX{\infdivx}[2]{[}{]}{%
  #1\;\delimsize\|\;#2%
}
\newcommand{\kl}{D_{KL}\infdivx}

\newcommand{\vx}{\bm x}
\newcommand{\vy}{\bm y}

\newcommand{\vg}{\bm g}

\newcommand{\half}{\nicefrac{1}{2}}
\newcommand{\eps}{\bm \epsilon}

\newcommand{\bz}{\bm z}
\newcommand{\be}{\begin{eqnarray} \begin{aligned}}
\newcommand{\ee}{\end{aligned} \end{eqnarray} }
\newcommand{\benn}{\begin{eqnarray*} \begin{aligned}}
\newcommand{\eenn}{\end{aligned} \end{eqnarray*} }
\newcommand{\xhat}{\hat{\bm x}}

\newcommand{\epshat}{\hat \eps}
\newcommand{\snr}{\gamma}
\newcommand{\rootsnr}{\sqrt{\gamma}}
\newcommand{\logsnr}{\alpha}
\newcommand{\vz}{\bm z_\snr}

\newcommand{\ds}{\frac{d}{d \snr}}

\newtheorem{theorem}{Theorem}[section]
\newtheorem{lemma}[theorem]{Lemma}

\DeclareMathOperator{\mse}{mse}
\DeclareMathOperator{\mmse}{mmse}
\DeclareMathOperator{\Mean}{Mean}
\DeclareMathOperator{\Var}{Var}

\DeclareMathOperator*{\argmin}{arg\,min}

\newcommand{\norm}[1]{{\| #1 \|_2^2 }}

\usepackage{color}

\title{Information-Theoretic Diffusion}
\author{Xianghao Kong \\ University of California Riverside  %
\And 
Rob Brekelmans \\ Vector Institute %
\And 
Greg Ver Steeg \\ University of California Riverside %
}

\date{}                                           %

\begin{document}
\maketitle

\begin{abstract}
Denoising diffusion models have spurred significant gains in density modeling and image generation, precipitating an industrial revolution in text-guided AI art generation.  We introduce a new mathematical foundation for diffusion models inspired by classic results in information theory that connect Information with Minimum Mean Square Error regression, the so-called I-MMSE relations. We generalize the I-MMSE relations to \emph{exactly} relate the data distribution to an optimal denoising regression problem, leading to an elegant refinement of existing diffusion bounds.  This new insight leads to several improvements for probability distribution estimation, including theoretical justification for diffusion model ensembling. Remarkably, our framework shows how continuous and discrete probabilities can be learned with the same regression objective, avoiding domain-specific generative models used in variational methods. Code to reproduce experiments is provided at \url{https://github.com/kxh001/ITdiffusion} and simplified demonstration code is at \url{https://github.com/gregversteeg/InfoDiffusionSimple}.
\end{abstract}

\section{Introduction}

Denoising diffusion models~\citep{jaschaneq} incorporating recent improvements~\citep{ddpm} now outperform GANs for image generation~\citep{diffusion_gan}, and also lead to better density models than previously state-of-the-art autoregressive models~\citep{vdm}. 
The quality and flexibility of image results have led to major new industrial applications for automatically generating diverse and realistic images from open-ended text prompts \citep{dalle2,imagen,latent_diffusion}. 
Mathematically, diffusion models can be understood in a variety of ways: as classic denoising autoencoders~\citep{vincent2011connection} with multiple noise levels and a new architecture~\citep{ddpm}, as VAEs with a fixed noising encoder~\citep{vdm}, as annealed score matching models~ \citep{song2019generative}, as a non-equilibrium process that tractably bridges between a target distribution and a Gaussian~\citep{jaschaneq}, or as a stochastic differential equation that does the same~\citep{diffusion_sde,liu2022let}.

In this paper, we call attention to a connection between diffusion models and a classic result in information theory relating the mutual information to the Minimum Mean Square Error (MMSE) estimator for denoising a Gaussian noise channel~\citep{guo}.
Research on Information and MMSE (often referred to as I-MMSE relations) transformed information theory with new representations of standard measures leading to elegant proofs of fundamental results~\citep{mmse_entropy}. This paper uses a generalization of the I-MMSE relation to discover an exact relation between data probability distribution and optimal denoising regression. The information-theoretic formulation of diffusion simplifies and improves existing results. Our main contributions are as follows. 
\begin{itemize}
    \item A new, exact relation between probability density and the global optimum of a mean square error denoising objective of the form, $-\log p(\vx) = \half \int_0^\infty \mmse(\vx, \snr) d\snr + \mbox{constant terms}$, with useful variations in Eq.~\eqref{eq:density}, \eqref{eq:density_simple} and \eqref{eq:density_data}. Because this expression is exact and not a bound, we can convert any statements about density functionals 
    such as entropy
    into statements about the unconstrained optimum of regression problems easily solved with neural nets. 
    \item The same optimization problem can also be exactly related to \emph{discrete} probability mass, Eq.~\eqref{eq:discrete}. In contrast, the variational diffusion bound requires specifying a separate discrete decoder term in the generative model and associated variational bound.  
    Our unification of discrete and continuous probabilities justifies empirical work applying diffusion to categorical variables.
    \item In experiments, we show that our approach can take pre-trained discrete diffusion models and re-interpret them as continuous density models with competitive log-likelihoods. Our approach allows us to fine-tune and \emph{ensemble} existing diffusion models to achieve better NLLs.  
\end{itemize}

\section{Fundamental Pointwise Denoising Relation}

Let $p(\vz |\vx)$ be a Gaussian noise channel with $\vz = \rootsnr \vx + \eps$ and $\eps \sim \mathcal N(0, \mathbb I)$, where $\snr$ represents the Signal-to-Noise Ratio (SNR) 
and $p(\vx)$ is the unknown data distribution. 
This channel has exploding variance as the SNR increases but we will see that the variance of this channel can be normalized arbitrarily without affecting results. Our convention matches the information theory literature and significantly simplifies proofs. 

The seminal result of \citet{guo} connects mutual information with MMSE estimators, 
 \begin{equation}\label{eq:immse} 
 \ds I(\vx; \vz) = \half \mmse(\snr)  .
\end{equation}
Here, the MMSE refers to the Minimum Mean Square Error for recovering $\vx$ in this noisy channel,
\begin{equation}\label{eq:mmse} 
\mmse(\snr) \equiv \min_{\xhat(\vz, \snr)} \mathbb E_{p(\vz, \vx)} \big[ 
\norm{\vx- \xhat(\vz, \snr)}
\big].
\end{equation}
We refer to $\xhat$ as the denoising function. 
The optimal denoising function  $\xhat^*$ corresponds to the conditional expectation, which can be seen using variational calculus or from the fact that the squared error is a Bregman divergence (
\citet{banerjee2005clustering}
Prop. 1), 
\begin{equation}\label{eq:opt}
\xhat^*(\vz, \snr) \equiv \argmin_{\xhat(\vz, \snr)} \mse(\snr)  = \mathbb E_{\vx \sim p(\vx |\vz) }[ \vx]
\end{equation}
The analytic solution is typically intractable because it requires sampling from the posterior distribution of the noise channel. 

Recent work on variational diffusion models writes
a lower bound on log likelihood
explicitly in terms of MMSE~\citep{vdm}, suggesting a potential connection to Eq.~\eqref{eq:immse}. However, the precise nature of the connection is not clear because the variational diffusion bound is an inequality while Eq.~\eqref{eq:mmse} is an equality, and the variational bound is formulated pointwise for $\log p(\vx)$ at a single $\vx$, while Eq.~\eqref{eq:mmse} is an expectation.

We now introduce a point-wise generalization of Guo et al's result that is the foundation for all other new results in this paper. %
  \begin{empheq}[box=\fbox]{equation}\label{eq:pimmse} 
 \ds \kl{p(\vz|\vx)}{p(\vz)}  = \half \mmse(\vx, \snr)  
  \end{empheq}
The marginal is $p(\vz) = \int p(\vz|\vx) p(\vx) d\vx$, and the pointwise MMSE is defined as follows,
\begin{equation}\label{eq:pmmse} 
\mmse(\vx, \snr) \equiv \mathbb E_{p(\vz|\vx)} \big[ 
\norm{\vx - \xhat^*(\vz, \snr)} 
\big] .
\end{equation}
Pointwise MMSE is just the MMSE evaluated at a single point $\vx$, and $\mathbb E_{p(\vx)}[\mmse(\vx, \snr)] = \mmse(\snr)$. 
Taking the expectation with respect to $\vx$ of both sides of Eq.~\eqref{eq:pimmse} recovers Guo et al's famous result, Eq.~\eqref{eq:immse}. 
Our proof of Eq.~\eqref{eq:pimmse} uses special properties of the Gaussian noise channel and repeated application of integration by parts, with a detailed proof given in Appendix~\ref{app:proof1}. This result can also be seen as a special case of Theorem 5 from \citet{mmse_gradient} by replacing their general channel with our channel written in terms of $\snr$ and then using the chain rule.   
In the rest of this paper, we show how this 
more general 
version of the I-MMSE relation can be used to reformulate denoising diffusion models.

\section{Diffusion as Thermodynamic Integration}

We first use Eq.~\eqref{eq:pimmse} to derive an expression for log-likelihood that resembles the variational bound. 
However, using results from the information theory literature~\citep{mmse_entropy}, we find that the expression can be significantly simplified and certain terms that are typically estimated can be calculated analytically. 

Our derivation is inspired by the recent development of thermodynamic variational inference~\citep{tvo,brekelmans2020all} which shows how to construct a path connecting a tractable distribution to some target, such that integrating over this path recovers the log likelihood for the target model. The paths between distributions can be generalized in a number of ways~\citep{qpath, chen2021variational}. 
Typically, though, these path integrals are difficult to estimate because they require expensive sampling from complex intermediate distributions. 
A distinctive property of diffusion models is that the Gaussian noise channel, which transforms the target distribution into a standard normal distribution, can be easily sampled at any intermediate noise level. 

We use the fundamental theorem of calculus to evaluate a function at two points in terms of the integral of its derivative, 
$\int_{\snr_0}^{\snr_1} d\snr \ds f(\snr)  = f(\snr_1) - f(\snr_0).$
This approach is known as ``thermodynamic integration'' in the statistical physics literature \citep{ogata1989monte, gelman1998simulating}, where evaluation of the endpoints corresponds to a difference in the free energy or log partition function, and the derivatives of these quantities may be more amenable to Monte Carlo simulation.

Consider applying the thermodynamic integration trick to the following function, $f(\vx, \snr) \equiv \kl{p(\vz|\vx)}{p(\vz)}.$ As before, we have $p(\vx)$, the data distribution, and a Gaussian noise channel, $\vz = \rootsnr \vx + \eps$ with different signal-to-noise ratios  $\snr$, 
\newcommand{\vza}{\bm z_{\snr_0}}
\newcommand{\vzb}{\bm z_{\snr_1}}
\begin{align*}
\int_{\snr_0}^{\snr_1} d\snr \ds f(\vx, \snr)  &=  \kl{p(\vzb | \vx)}{p(\vzb)} - {\color{blue} \kl{p(\vza |\vx)}{p(\vza )} } \\
&=      \kl{p(\vzb | \vx)}{p(\vzb)}   -   {\color{blue} \mathbb E_{p(\vza | \vx)} [\log p(\vx | \vza)] + \log p(\vx) } .
\end{align*}
In the second line, we expanded the KL divergence and used Bayes rule. Next, we re-arrange and use Eq.~\eqref{eq:pimmse} to re-write the integrand 
\begin{equation}\label{eq:var}
\hspace*{-.15cm} -\log p(\vx) =  \underbrace{\kl{p(\vzb |\vx)}{p(\vzb )}
\vphantom{E_{p(\vza | \vx)}}
}_{\text{Prior loss}} + \underbrace{\mathbb E_{p(\vza | \vx)} [-\log p(\vx | \vza)]
}_{\text{Reconstruction loss}}   \underbrace{- \half \int_{\snr_0}^{\snr_1} \mmse(\vx, \snr) d\snr}_{\text{Diffusion loss}}.
\end{equation}
Comparing to a particular variational bound for diffusion models in the continuous time limit (Eq. 15 from \citet{vdm}), we see this expression looks similar (see App.~\ref{app:variational} for more details). However, our derivation so far is exact and we haven't introduced any variational approximations. 
Prior and reconstruction loss terms are stochastically estimated in the variational formulation, but we show this is unwise and unnecessary. Consider the limit where $\snr_1 \rightarrow 0$. In that case, the prior loss will be zero. 
The reconstruction term, for continuous densities, becomes infinite in the limit of large $\snr_0$. For this reason, recent diffusion models only consider reconstruction for discrete random variables, $P(\vx | \vza)$.  %
Estimation of conflicting divergent terms and the need for separate prior and reconstruction objectives can be avoided, as we now show.

\textbf{Simple and exact probability density via MMSE~~~~~~} 
We consider applying thermodynamic integration to a slightly different function, and expand the range of integration to $\snr \in [0,\infty)$. 
Consider sending samples from either the data distribution $p(\vx)$ or a standard Gaussian $p_G(\vx) = \mathcal N(\vx; 0, \mathbb I)$ through our Gaussian noise channel.   We denote the MMSE for the channel with Gaussian input as $\mmse_G(\snr)$, and write its marginal output distribution as $p_G(\vz) = \int p(\vz|\vx) p_G(\vx) d\vx$.
Finally, we define the function $f(\vx, \snr)$ as
\footnote{Note that $\mathbb{E}_{p(\vx)}[f(\vx, \snr)] = \mathbb{E}_{p(\vx,\vz)}[\log {p(\vz)} - \log {p_G(\vz)}] = D_{KL}[p(\vz) \| p_G(\vz)]$ is the gap in an upper bound $\mathbb{E}_{p(\vx)}[ D_{KL}[p(\vz|\vx)\| p_G(\vz)]]$ on mutual information $I(\vx;\vz)$, where the upper bound uses the marginal distribution induced by the Gaussian source $p_G(\vz)$ instead of the data $p(\vz)$.}
$$f(\vx, \snr) \equiv \kl{p(\vz|\vx)}{p_G(\vz)} - \kl{p(\vz|\vx)}{p(\vz)}.$$
In the limit of zero SNR, we get $\lim_{\snr \rightarrow 0} f(\vx, \snr) = 0$. In the high SNR limit we use the following result proved in App.~\ref{app:proof2}, 
\begin{equation}\label{eq:log_ratio}
\lim_{\snr \rightarrow \infty} f(\vx, \snr) = \log \frac{p(\vx)}{p_G(\vx)}.
\end{equation}
Combining this with Eq.~\eqref{eq:pimmse}, we can write the log likelihood \emph{exactly} in terms of the log likelihood of a Gaussian and a one dimensional integral. 
\begin{align}\label{eq:density}
-\log p(\vx) &= -\log p_G(\vx) - \int_{0}^{\infty} d\snr \ds f(\vx, \snr) \nonumber \\
&= { -\log p_G(\vx)} - \half \int_{0}^{\infty} d\snr \left( {\mmse_G(\vx, \snr)} - \mmse(\vx, \snr) \right) 
\end{align}
This expresses density in terms of a Gaussian density and a correction that measures how much better we can denoise the target distribution than we could using the optimal decoder for Gaussian source data. 
The density can be further simplified by writing out the Gaussian expressions explicitly and simplifying with an identity given in App.~\ref{app:identity}.
  \begin{empheq}[box=\fbox]{equation}\label{eq:density_simple} 
-\log p(\vx) = d/2 \log(2 \pi e) - \half \int_{0}^{\infty} d\snr \left(  \frac{d}{1+\snr} - \mmse(\vx, \snr) \right)
  \end{empheq}
  This expression shows that density can be written solely in terms of the global optimum of a particular regression problem, the denoising MSE. This is convenient because neural networks excel at unconstrained optimization of MSE loss functions. 
  If we take the expectation of $-\log p(\vx)$ using Eq.~\eqref{eq:density_simple}, we recover a relatively recently discovered representation of the differential entropy, $h(p)$, in terms of MMSE~\citep{mmse_entropy}. This highlights an advantage of our approach, as all density functionals can be 
rewritten in terms of the solution of an unconstrained regression problem,
\begin{equation} \label{eq:entropy}
 h(p) \equiv \mathbb E_{p(\vx)} [-\log p(\vx)]  = d/2 \log 2 \pi e  - \half \int_{0}^{\infty} d\snr \left( \frac{d}{1+\snr} - \mmse(\snr) \right) .
 \end{equation}

  Since the first integral in Eq.~\eqref{eq:density_simple}-\eqref{eq:entropy} does not depend on $\vx$,  it is tempting to absorb it into a constant.  
  However, the first integral diverges, and the second integral typically divergences as well -- only the difference converges.
This observation will help improve density estimation, by noticing that only the gap between MMSEs is important and that the gap becomes small at high and low values of SNR.

  \begin{figure}[htbp] %
   \centering
   \includegraphics[width=0.37\columnwidth]{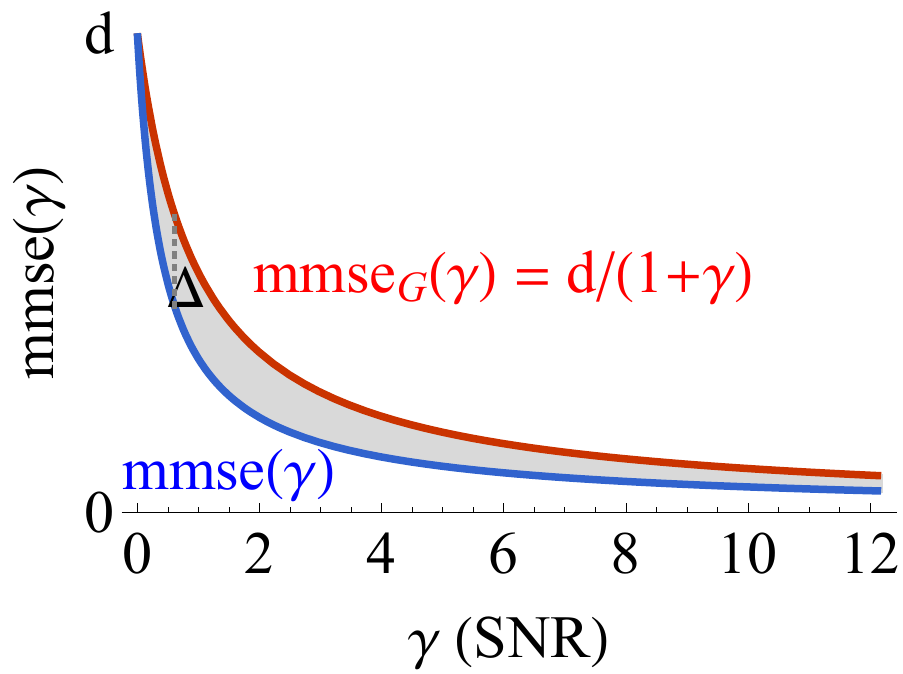} \qquad 
      \includegraphics[width=0.37\columnwidth]{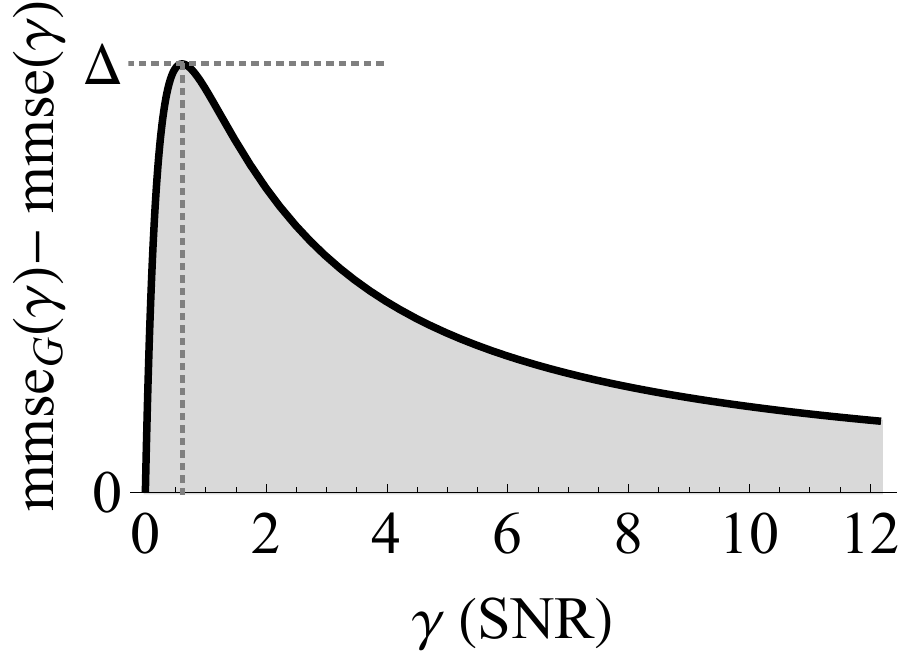} 
   \caption{The integral of the gap between MMSE curves for data from the target distribution versus data from a Gaussian distribution is used in Eq.~\eqref{eq:entropy} to get an exact expression for the entropy, or expected Negative Log Likelihood (NLL), of the data.}
   \label{fig:mmse}
\end{figure}
  
   We show example MMSE curves for denoising with Gaussian input versus non-Gaussian inputs in Fig.~\ref{fig:mmse}. 
   The goal is to integrate the gap between these two curves, one analytic (for Gaussians) and one given by the data. 
 Note that the gap could in principle be positive or negative. However, we know that Gaussians have the maximum possible MMSE, compared to any distribution with the same variance or covariance, when used as an input to a Gaussian noise channel~\citep{shamai1992worst}. 
Therefore, if we always compare to Gaussians with the same variance or covariance as the data, we can guarantee that the gap is positive. 
In practice, if the gap ever becomes negative, which we show can occur for models appearing in the literature in Sec.~\ref{sec:experiments}, it means that the discovered estimator is sub-optimal and we can fall back to the optimal Gaussian estimator to improve results in this region.

We also see that for large and small SNR values, the gap becomes small. At low SNR, the noisy data is approximately Gaussian.  At high SNR, the noise goes to zero, so a linear (Gaussian) denoiser is sufficient in either case. 
Recognizing that the important signals come from intermediate SNR values can help save computation over approaches that indiscriminately optimize over the whole curve, as we discuss in Sec.~\ref{sec:implementation}.

Standardizing data to have unit variance to ensure that the MMSE gap is always positive may be inconvenient, so we derive a variation of Eq.~\eqref{eq:density_simple} that takes the scale into account. Instead of using $p_G(\vx) = \mathcal N(\vx; 0, \mathbb I)$, we can take the base measure to be a Gaussian, $p_G(\vx) = \mathcal N(\vx; \bm \mu, \bm \Sigma)$, with mean and covariance that match the data. 
Letting $\lambda_1,\ldots, \lambda_d$ be the eigenvalues of the covariance matrix $\bm \Sigma$, we derive a compact expression 
in App.~\ref{app:density_data},
\begin{equation}\label{eq:density_data} 
\mathbb E_{p(\vx)} [-\log p(\vx)] = \underbrace{1/2 \log \det(2 \pi e \bm \Sigma)}_{\text{Gaussian entropy}} - \underbrace{\half \int_{0}^{\infty} d\snr \left(  \sum_{i=1}^d \frac{1}{\snr+1/\lambda_i} - \mmse(\snr) \right)}_{\text{Deviation from Gaussianity} \geq 0}.
\end{equation}
This expression tells us precisely which part of the reconstruction mean square error curve is actually important, namely the part that deviates most from Gaussianity. 
Note that if the eigenvalues of $\bm \Sigma$ are not feasible to estimate, we can use a diagonal covariance matrix for the base Gaussian
and still preserve the desired property that the gap between the data MMSE and Gaussian MMSE is non-negative~\citep{shamai1992worst}.

Finally, we make several remarks to relate our exact expression for the likelihood 
in Eq.~\eqref{eq:density_simple} to existing work in the diffusion literature.
\begin{itemize}
\item Since the MSE minimization 
is performed at each $\snr$,  reweighting terms with different $\snr$~\citep{vdm} should not have an effect as long as we achieve the global minima at each $\snr$.   
\item Several heuristics have been suggested for improving the ``noise schedule'', which corresponds to points for evaluating the MMSE integral in our formulation~\citep{ddpm,vdm,nichol2021improved}.  Our formulation highlights the importance of sampling from SNRs where there is a large gap, and we suggest a simple and effective strategy for this in Sec.~\ref{sec:implementation}. 
 \item The literature on diffusion models also suggests more complex denoising distributions that also model the covariance ~\citep{nichol2021improved}.   Modeling the covariance of the denoising model is unnecessary in our approach, as our exact expressions never 
 require it. 
 \item  Compare
   Eq.~\eqref{eq:density_simple} to the exact, continuous density expression in Eq. 39 of \citet{diffusion_sde}, which combines neural ODE flows and diffusion models.   The form of that expression is: 
$$-\log p(\vx(0)) = -\log p_G(\vx(T)) + \int_0^T \nabla \cdot \vg(\vx(t), t) dt.$$ 
The first step to using this expression is to solve a differential equation for the full trajectory, $\vx(t)$, that depends on many evaluations of a learned denoising diffusion (or score) model. Then, estimating the integral is highly non-trivial as $\vg$ is complex and its divergence is expensive to compute, necessitating some stochastic approximations.  See Sec. \ref{sec:related} for additional discussion.
 \end{itemize}

\textbf{Discrete probability estimator~~~~~~} 
Modeling probability distributions over continuous versus discrete random variables typically requires rather different approaches. 
An unusual feature of I-MMSE relations is that they naturally handle both types of random variables. 
In this subsection, consider that $\vx \in \mathcal X \subset \mathcal Z^d$, with a domain that is discrete but numeric (non-numeric discrete variables can be handled via an appropriate mapping function~\citep{guo}). We will use capital $P(\vx)$ to denote probability (rather than probability \emph{density}) over this discrete domain. Note that the Gaussian noise channel, $\vz = \rootsnr \vx + \epsilon$, is still continuous with an associated probability density, $p(\vz|\vx)$.

Re-doing the analysis in Eq.~\eqref{eq:var} leads to the same expression, but replacing $p(\vx)$ with $P(\vx)$. The first two terms in Eq.~\eqref{eq:var} go to zero, 
$$ \lim_{\snr_0 \rightarrow \infty, \snr_1 \rightarrow 0}  \kl{p(\vzb |\vx)}{p(\vzb )} + \mathbb E_{p(\vza | \vx)} [-\log P(\vx | \vza)]  = 0 .$$
This leads to the following form for the discrete probability. 
\begin{empheq}[box=\fbox]{equation}\label{eq:discrete}
-\log P(\vx) = \half \int_{0}^{\infty} \mmse(\vx, \snr) d\snr
\end{empheq}
At large SNR, $\vz$ is very concentrated around a discrete point specified by $\vx$. 
In that case, we should be able to recover the true value from the slightly perturbed value with high probability. 
Sec.~\ref{sec:implementation} derives tail bounds used for approximating this integral numerically. 
Taking the expectation of this general result recovers an equation for the Shannon entropy from Section VI of \citet{guo}, 
\begin{align}\label{eq:discrete_entropy}
H(\vx) \equiv - \sum_{\vx \in \mathcal X} P(\vx) \log P(\vx) = \mathbb E_{\vx \sim P(\vx)}[ -\log P(\vx)] = \half \int_{0}^{\infty} \mmse(\snr) d\snr.
\end{align}
Conveniently, our discrete and continuous probability representations are nearly identical -- they rely on the same integral and optimization and differ only by constant terms.  
We discuss comparisons to the variational bounds in App.~\ref{app:variational} and numerical implementation details in Sec. \ref{sec:implementation}.

\section{Implementation}\label{sec:implementation}

While we start with exact expressions for density and discrete probability in Eqs.~\eqref{eq:density_simple} and \eqref{eq:discrete}, in practice there will be two sources of error when implementing our approach numerically. First, we must parametrize our estimator as a neural network that may not achieve the global minimum mean square error and, second, we are forced to rely on numerical integration.

\textbf{MMSE Upper Bounds~~~~~~}
While the likelihood bounds in Eqs.~\eqref{eq:var}, \eqref{eq:density}, and \eqref{eq:density_simple} are exact, evaluating each $\mmse(\vx,\snr)$ term requires access to the optimal denoising function or conditional expectation $\xhat^*(\vz, \snr) = \mathbb{E}_{p(\vx|\vz)}[\vx]$.     
Using a suboptimal denoising function $\xhat(\vz, \snr)$ instead of $\hat{\vx}^*(\vz,\snr)$ to approximate the MMSE, we obtain an upper bound whose gap can be characterized as
\begin{align*}
\mathbb{E}_{p(\vx,\vz)}\big[
\norm{\vx - \hat{\vx}(\vz,\snr)} 
\big] = \underbrace{ \mathbb{E}_{p(\vx,\vz)}\big[
\norm{\vx - \hat{\vx}^*(\vz,\snr)}
\big] }_{\mmse(\snr)} + 
\underbrace{
\mathbb{E}_{p(\vz)}\big[\norm{\hat{\vx}^*(\vz,\snr) - \hat{\vx}(\vz,\snr) }\big]
}_{\text{estimation gap}}
.
\end{align*}
In App. \ref{app:bregman}, we derive this upper bound and its gap using results of \citet{banerjee2005clustering} for general Bregman divergences.
This translates to 
an upper bound on the Negative Log Likelihood (NLL),
\begin{equation*}
  \mathbb{E}_{p(\vx)}[-\log p(\vx)] \leq {\half \log \det(2 \pi e \bm \Sigma)} - \half \int_0^\infty d\snr \left( 
  \sum_{i=1}^d \frac{1}{\snr+1/\lambda_i}
  -  \mathbb{E}_{p(\vx,\vz)}\big[\norm{\vx - \hat{\vx}(\vz,\snr) }\big] \right).
\end{equation*}
\normalsize

\textbf{Restricting Range of Integration~~~~~~}
As noted in Fig. \ref{fig:mmse}, only the difference $\mmse_G(\snr) - \mmse(\snr) \geq 0$ contributes to our expected NLL bound.  The non-negativity of this difference is guaranteed by the results of \citet{shamai1992worst}, but may not hold in practice due to our suboptimal estimators of the MMSE, $\mathbb{E}_{p(\vx,\vz)}[\norm{\vx - \hat{\vx}(\vz,\snr) }] \geq \mmse(\snr)$.

In regions $\snr \leq \snr_0$ or $\snr \geq \snr_1$ where the difference in mean square error terms appears to be negative, we can simply define $\xhat(\vz, \snr) \equiv \xhat^*_G(\vz, \snr)$,
where $\xhat^*_G(\vz, \snr)$ is the optimal, linear decoder for a Gaussian with the same covariance as the data (App. \ref{app:gaussian}).
For this choice of decoder, the integrand becomes zero so that we may drop the tails of the integral outside of the appropriate range $\snr \in [\snr_0, \snr_1]$,
\begin{equation*}\label{eq:lossx} 
 \mathbb E_{p(\vx)} [-\log p(\vx)] \leq \half \log \det(2 \pi e \bm \Sigma) - 
 \half \int_{\snr_0}^{\snr_1} d\snr \left(  \sum_{i=1}^d \frac{1}{\snr+1/\lambda_i} - \mathbb E_{p(\vx, \vz)}[\norm{\vx - \xhat(\vz, \snr)}] \right).
\end{equation*}

\textbf{Parametrization~~~~~~}
We parametrize $\xhat(\vz, \snr) \equiv (\vz - \epshat(\vz, \snr)) / \rootsnr, \forall \snr \in [\snr_0, \snr_1]$.   With this definition, we can re-arrange to see that the error for reconstructing the noise, $\eps$, is related to the error for reconstructing the original image, $\epshat(\vz, \snr) - \eps = \rootsnr(\vx -  \xhat(\vz, \snr))$. Our neural network then implements $\epshat$, a function that predicts the noise. The inputs to this network are $(\vz, \snr)$. Note that we can equivalently parametrize our network to use the inputs $(\vz / \sqrt{1+\snr}, \snr)$. This choice is preferable so that the variance of the noisy image representation is bounded, while the SNR is unchanged. Numerically, it is better to work with log SNR values, so we change variables, $\logsnr = \log \snr$, which leads to the following form 
\begin{equation} \label{eq:losseps} 
 \mathbb E_{p(\vx)} [-\log p(\vx)] \leq \half~ \log \det(2 \pi e \bm \Sigma) - \half \int_{\logsnr_0}^{\logsnr_1} d\logsnr  \left( f_\Sigma(\logsnr) - \mathbb E_{\vx,\eps}\big[\norm{\eps - \epshat(\vz, \snr)}\big] \right).
\end{equation}
Here, $f_\Sigma(\logsnr) \equiv \sum_{i=1}^d \sigma(\logsnr + \log \lambda_i)$, using the traditional sigmoid function, $\sigma(t) = 1/(1+e^{-t})$.

\textbf{Numerical Integration~~~~~~} The last technical issue to solve is how to evaluate the integral in Eq.~\eqref{eq:losseps} numerically. We use importance sampling to write this expression as an expectation over some distribution, $q(\logsnr)$, for which we can get unbiased estimates via Monte Carlo sampling. This leads to the final specification of our loss function $\mathbb E_{p(\vx)} [-\log p(\vx)] \leq \mathcal L$, where
\begin{equation} \label{eq:loss} 
 \mathcal L \equiv \half~ \log \det(2 \pi e \bm \Sigma) - \half~ \mathbb E_{q(\logsnr)} \left[ 1/q(\logsnr) \left( f_\Sigma(\logsnr) - \mathbb E_{\vx,\eps}\big[\norm{\eps - \epshat(\vz, \snr)}\big] \right) \right] .
\end{equation}
All that remains is to choose $q(\logsnr)$. We use our analysis of the Gaussian case to motivate this choice. Note that $f_\Sigma(\logsnr)$ is a mixture of CDFs for logistic distributions with unit scale and different means. Mixtures of logistics with different means are well approximated by a logistic distribution with a larger scale~\citep{crooks2009logistic}. So we take $q(\alpha)$ to be a truncated logistic distribution with mean $\mu = \Mean_i(-\log \lambda_i)$ and scale $s = \sqrt{1+ 3/\pi^2 \Var_i(\log \lambda_i)}$, based on moment matching to the mixture of logistics implied by $f_\Sigma(\logsnr)$. Samples can be generated from a uniform distribution using the quantile function, $\alpha = \mu + s \log t/(1-t), t \sim \mathcal U[t_0, t_1]$. The quantile range is set so that $\alpha \in [\mu - 4 s, \mu + 4 s]$. 
The objective is estimated with Monte Carlo sampling, providing a stochastic, unbiased estimate of our upper bound in Eq. \eqref{eq:losseps}. 
Numerical integration alternatives include the trapezoid rule~\citep{lartillot2006computing, friel2014improving, hug2016adaptive} or a Riemann sum, due to the monotonicity of $\mmse(\vx,\snr)$ (\citet{vdm} Fig. 2, \citet{tvo}).

\textbf{Comparing between continuous and discrete probability estimators~~~~~~}
Recent diffusion models are trained assuming discrete data. We can measure how well they model continuous data by viewing continuous density estimation as the limiting density of discrete points~\citep{jaynes}. 
Treating $p(\vx)$ as a uniform density in some $d$-dimensional box or bin of volume $\Delta^d$ around each discrete point leads to the relation, 
$\mathbb E[-\log p(\vx)] = \mathbb E[-\log P(\vx \in \text{ bin}) / (\Delta)^d] = \mathbb E[-\log P(\vx \in \text{ bin}) - d \log \Delta]. $
In the other direction, when we use a continuous density estimator to model a discrete density, we use uniform dequantization~\citep{ho_dequantize}.

\textbf{Direct discrete probability estimate~~~~~~} 
Finally, we derive a practical upper bound on negative log likelihood for discrete probabilities. 
\begin{align}
\mathbb E[-\log P(\vx)] &= 1/2 \int_0^\infty \mmse(\snr) d\snr =  1/2 \int_{\snr_0}^{\snr_1} \mmse(\snr) d\snr + {\color{blue} 1/2 \left(\int_{0}^{\snr_0} + \int_{\snr_1}^\infty \right) \mmse(\snr) d\snr}  \nonumber \\
\mathbb E[-\log P(\vx)] &\leq 1/2 \int_{\snr_0}^{\snr_1} \mathbb E_{\vz,\vx}\big[ \norm{\vx - \xhat(\vz, \snr)} \big] d\snr + {\color{blue} c(\snr_0, \snr_1)} \label{eq:final_discrete}
\end{align}
In App.~\ref{app:discrete}, we analytically derive upper bounds for the left and right tail of the integral, expressed using $c(\snr_0, \snr_1)$. 
We also get an upper bound from using a denoiser that is not necessarily globally optimal. 
The discrete and continuous density estimators differ only by constants, therefore we can use the same importance sampling estimator for the integral, as in Eq.~\ref{eq:loss}.

\section{Experiments} \label{sec:experiments}

Our results establish an exact connection between the data likelihood and the solution to a regression problem, the Gaussian denoising MMSE. If we integrate the MSE curve of a particular denoiser and get a worse bound, it simply means that the denoiser does not achieve the MMSE. 
Interestingly, variational diffusion models optimize an objective that can be made quite similar with appropriate choices for variational distributions (App.~\ref{app:variational}).
Therefore, we can take previously trained diffusion models and evaluate alternate likelihood bounds using the methods described above. 
We can attempt to improve these bounds by directly optimizing the denoising MSE. 
Finally, a straightforward implication of our results is that we would like the minimum MSE denoiser at each SNR value; therefore, if we have different denoisers with lower error at different SNR values, we can combine them to get a density estimator which outperforms any individual one. \emph{Ensembling} diffusion models that ``specialize'' in different SNR ranges can likely be used to construct new SOTA density models. 

For the following experiments we use the CIFAR-10 dataset, scaled to have pixel values in $[-1,1]$, and consider a pre-trained DDPM model~\citep{ddpm} and an Improved DDPM we refer to as IDDPM~\citep{nichol2021improved}. 
We use 4000 diffusion steps for calculating variational bounds. For our IT based methods, the comparable parameter is how many log SNR samples, $\logsnr \sim q(\logsnr)$, we use for evaluation. For continuous density estimation 100 points were sufficient, while for the discrete estimator we used 1000 points (as discussed in App.~\ref{app:variance}).
More details on models and training can be found in App.~\ref{app:tune}. 

\begin{table}[tb]
	\begin{minipage}{0.55\linewidth}
		\centering
		\includegraphics[width=0.97 \columnwidth]{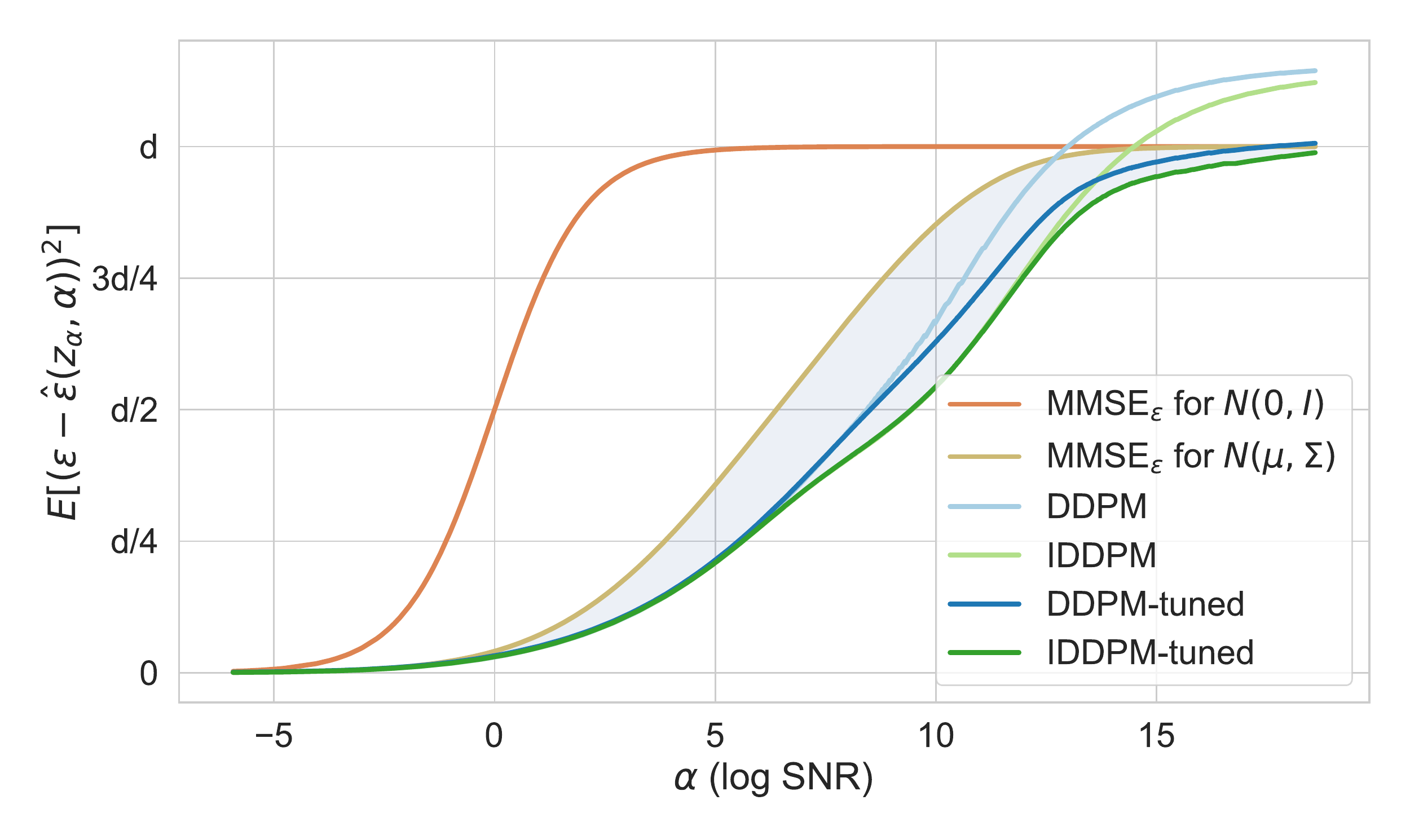}
	\end{minipage} \hfill
	\begin{minipage}{0.44\linewidth}
		\caption{$\mathbb E[-\log p(\vx)]$ on Test Data (bpd)}
		\label{table:cont_nll}
		\centering
		\small
    \begin{tabular}{*4c}
        \toprule
        & \multicolumn{2}{c}{Variational Bounds} & IT \\
        \cmidrule(lr){2-3}  
        Model & Disc.*  & Cont. & Cont. \\    
        \midrule

        DDPM & \textbf{-3.62}  & -3.44 &  -3.56 \\
        IDDPM &  -4.05 & -3.58  & \textbf{-4.09} \\
        \midrule
        DDPM(tune)& -3.55 & -3.51 & \textbf{-3.84} \\
        IDDPM(tune) & -3.85 &-3.55  & \textbf{-4.28}\\
        \midrule
        Ensemble & - & - & \underline{\textbf{-4.29}}\\
        \bottomrule
    \end{tabular}
	\end{minipage}
	\captionof{figure}{(Left) MSE curves for different denoisers, used in estimating Negative Log Likelihood (NLL). (Right) Continuous NLL estimates for diffusion models using variational bounds and Information-Theoretic (IT) bounds (ours). Variance estimates are shown in App.~\ref{app:variance}. *Uses a discrete estimator made continuous by assuming uniform density in each bin.}\label{cont_density}\vspace{-5mm}
\end{table} 

\textbf{Continuous Density Estimation with Diffusion~~~~~}
We first explore the results of continuous probability density estimation on test data using variational bounds versus our approach. 
Recent diffusion based models treat the data as discrete, bounding $\log P(\vx)$, however, we could also use diffusion models to give continuous density estimates by interpreting the last denoising step as providing a Gaussian distribution over $\vx$.
For comparison, we can also take the variational bound for discrete data and create a density estimate by making density uniform within each bin.
The results are shown in Fig.~\ref{cont_density}. For our estimator, the shaded integral between the Gaussian MMSE and the denoising MSE is used, and this provides noticeably improved density estimates.  Note that for our bound calculated using the IDDPM model, we throw away the part of their network that estimates the covariance and still achieve improved results, validating our assertion that variational modeling of covariance is unnecessary. 

Continuous density estimators applied to intrinsically discrete pixel data, like CIFAR-10, can in principle lead to very low NLLs if the model learns to put large mass on discrete points. 
In our case, we are comparing the same model so the comparison would still be meaningful. Furthermore, we can see that the diffusion models are not putting probability mass on discrete points by comparing the MSE curves to denoisers where discretization is purposely introduced in Fig.~\ref{disc_density}.

\textbf{Discrete Probability~~~~~}
NLL bounds for diffusion models in recent work are for discrete probability estimates.   We compare to our estimate based on an exact relation between MMSE and discrete likelihood, Eqs.~\eqref{eq:discrete} and  \eqref{eq:final_discrete}, and 
to our continuous NLL estimator treated as a discrete estimator with uniform dequantization. 
The diffusion architectures we tried do not naturally concentrate predictions on grid points at high SNR, so we 
had to explicitly add rounding via a ``soft discretization'' function (described in App.~\ref{app:discretize}), which leads to much lower error at high SNR, as seen in Fig.~\ref{disc_density}.  Further improvements from ensembling the best models at different SNRs are described below.

\textbf{Fine-tuning~~~~~~} 
Fig.~\ref{cont_density} and \ref{disc_density} show that we can improve log likelihoods by fine-tuning existing models using our regression objective derived in Eq.~\ref{eq:loss}, rather than the variational bound. 
In particular, we see that the improvements for continuous density estimation come from reducing error at high SNR levels.
The inability of existing architectures to exploit discreteness in the data is the reason fine-tuning improves the continuous but not the discrete estimator in Table~\ref{table:disc_nll}.
The final discrete estimates are improved by including a soft discretization nonlinearity and using ensembling, as described next.  
See App.~\ref{app:tune} for additional experimental details. 

\textbf{Ensembling~~~~~~} 
Finally, we propose to ensemble different denoising diffusion models by choosing the denoiser with the lowest MSE at each SNR level in Eq. \eqref{eq:losseps}.   
Since our estimator depends only on estimating the smallest MSE at each $\snr$, our likelihood estimates will benefit from combining models whose relative performance differs across SNR regions.  
In Fig.~\ref{cont_density} and \ref{disc_density}, we report improved results from ensembling DDPM \citep{ddpm}, IDDPM \citep{nichol2021improved}, fine-tuned versions, and rounded versions (in the discrete case). For discrete estimators, rounding is useful at high SNR but counterproductive at low SNR. The ensemble MSE (shaded region of Fig.~\ref{disc_density}) gives a better NLL bound than any individual estimator. 
 
\begin{table}[tb]
	\begin{minipage}{0.55\linewidth}
		\centering
		\includegraphics[width=0.97 \columnwidth]{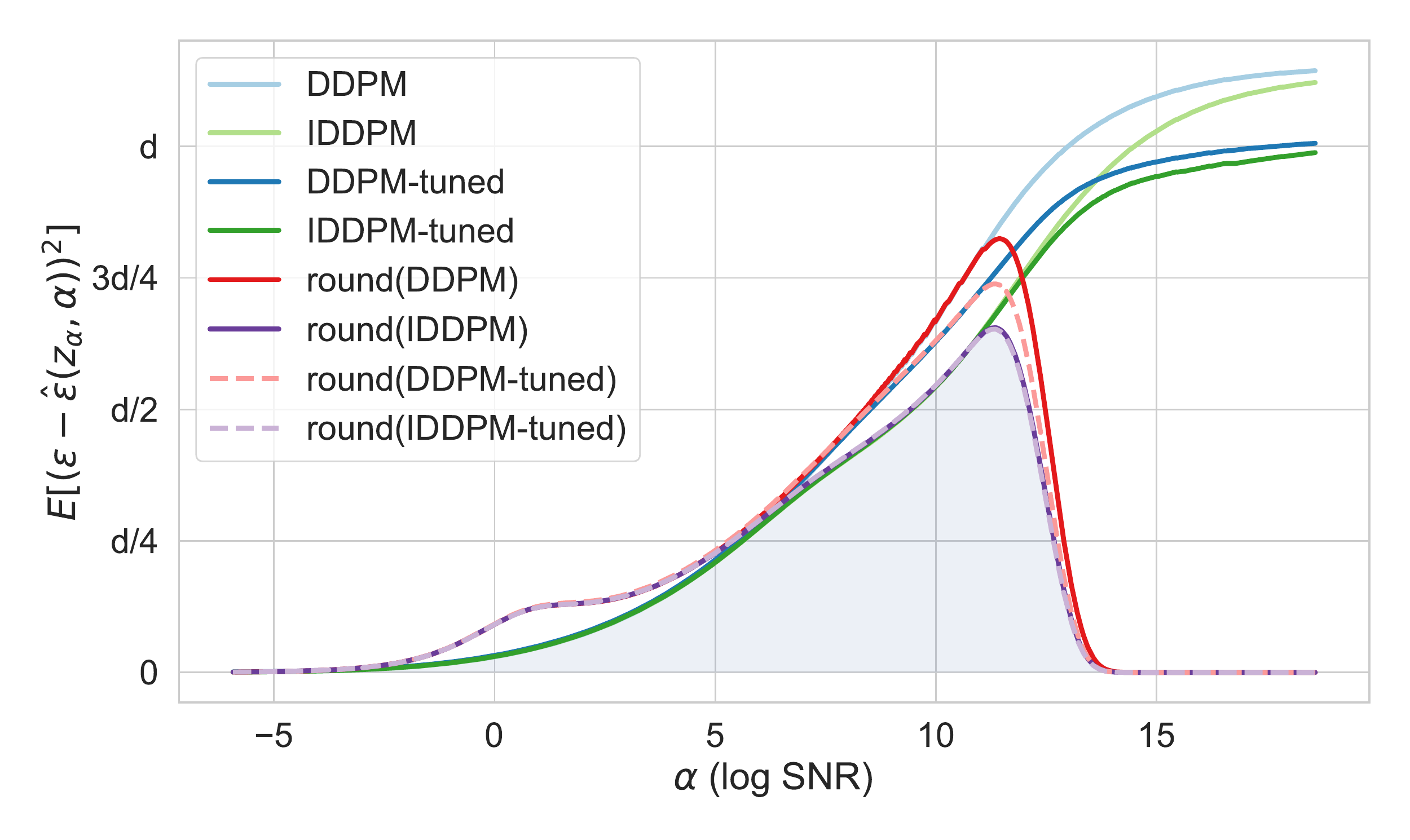}
	\end{minipage} \hfill
	\begin{minipage}{0.44\linewidth}
		\caption{$\mathbb E[-\log P(\vx)]$ on Test Data (bpd)}
		\label{table:disc_nll}
		\centering
		\small
    \begin{tabular}{*4c}
        \toprule
        & Var. &\multicolumn{2}{c}{IT} \\
        \cmidrule(lr){3-4}  
        Model & Disc.  & Disc.* & Cont.** \\    
        \midrule
DDPM & \textbf{3.37} & 3.68 & 3.51 \\
IDDPM & \textbf{2.94} & 3.16 & 3.17 \\
\midrule
DDPM(tune) & 3.45 & 3.48 & \textbf{3.41} \\
IDDPM(tune) & \textbf{3.14} & 3.15 & 3.18 \\
\midrule
Ensemble & - &  \underline{\textbf{2.90}} & 3.16 \\

        \bottomrule
    \end{tabular}
	\end{minipage}
	\captionof{figure}{(Left) MSE curves for different denoisers, used in estimating Negative Log Likelihood (NLL), shown with or without rounding using a soft discretization function (App.~\ref{app:discretize}). (Right) Discrete NLL estimates for diffusion models using variational bounds and information-theoretic bounds (ours). Variance estimates are shown in App.~\ref{app:variance}. *Using the denoiser with soft discretization. **Indicates a continuous estimator made discrete via uniform dequantization.}\label{disc_density}\vspace{-5mm}
\end{table}

\section{Related Work}\label{sec:related}

The developments in diffusion models that have enabled new industrial applications~\citep{dalle2,imagen,latent_diffusion} build on a rich intellectual history of ideas spanning many fields including score matching\citep{hyvarinen2005estimation}, denoising autoencoders~\citep{vincent2011connection}, nonequilibrium thermodynamics~\citep{jaschaneq}, neural ODEs~\citep{chen2018neural}, and stochastic differential equations~\citep{diffusion_sde,huang2021variational}. The observation that part of the variational diffusion loss could be written as an integral of MMSEs was made by~\citet{vdm}. 
To the best of our knowledge, the connection between I-MMSE relations in information theory and diffusion models has so far gone unrecognized.

\citet{diffusion_sde} introduced a different way to use diffusion to model probability densities via differential equations. This approach requires solving a differential equation of a dynamic trajectory for each sample. Our estimate can be computed in a single step by evaluating the score model at each SNR in parallel, but the neural probability flow ODE will require thousands of sequential score function evaluations to solve.
Concurrent work uses stochastic differential equations along with Girsanov's theorem, a stochastic version of the change of variables formula, to show that optimizing a regression objective is sufficient to exactly match a stochastic process to some target density~\citep{liu2022let,ye2022first}. The results are used to improved sampling, but they use standard variational bounds for density estimation. Our work, in contrast, exactly relates regression to density estimation but does not address sampling. The form of their results for discrete distributions suggests a deeper connection between their work and this one. 

Our work focused on density modeling, so approaches that forgo density modeling to improve image generation, for instance by doing diffusion in a latent space~\citep{latent_diffusion}, were not considered. \citet{cold_diffusion} observe that denoising diffusion models using many types of noise can lead to good image sampling, however, our results highlight the special connection between Gaussian denoising and density modeling. %

Recent work has studied applicability of diffusion to discrete and categorical variables ~\citep{austin2021structured,gu2022vector,analogbits}.
Our exact relation between denoising and discrete probability provides theoretical justification for this promising line of research.
Finally, this paper gives a theoretical justification for ensembling diffusion models, to achieve the best MSE at different SNR levels. We see from concurrent work that this strategy is already being successfully employed~\citep{balaji2022ediffi}. 

\section{Conclusion} 

Variational methods are powerful, but require many choices in designing variational approximations.
More complex choices for the variational distributions could lead to tighter bounds, and a fair amount of work has explored this possibility. 
Interestingly, though, the most successful refinements of diffusion models have led to objectives that are more similar to the denoising MSE, and the results in this paper finally make it clear that regression is all that is needed for exact probability estimation, for both continuous and discrete variables.  
We generalized the classic I-MMSE relation from information theory to introduce this simple, \emph{exact} relationship between the data probability distribution and optimal denoisers, $-\log p(\vx) = \half \int_0^\infty \mmse(\vx, \snr)~d\snr + \mbox{constant terms}$.  
This result pares away the unnecessary ingredients in the variational approach, and
the simple and evocative form of the result can inspire many improvements and generalizations to be explored in future work. 

\clearpage
\bibliographystyle{iclr2023_conference}
\bibliography{nsf}

\newpage
\appendix

\section{Derivations}

\subsection{Derivation of the Fundamental Pointwise Denoising Relation}\label{app:proof1}
We now derive the following pointwise denoising relation. 
  \begin{equation*} 
 \ds \kl{p(\vz|\vx)}{p(\vz)}  = \half \mmse(\vx, \snr)  
  \end{equation*}
Recall the definition of pointwise MMSE, except we will use a shorthand notation for the analytic MMSE denoiser. 
\begin{align*}
\mmse(\vx, \snr) &\equiv \mathbb E_{p(\vz|\vx)} [ \norm{\vx - \mathbb E[\vx | \vz] }]  \\
\mathbb E[\vx | \vz] &\equiv \int d\bar \vx~ p(\bar \vx | \vz)~\bar \vx = \int d\bar \vx~  p(\vz | \bar \vx) p(\bar \vx) / p(\vz) ~\bar \vx
\end{align*}
We begin by expanding the left-hand side. 
\begin{align*}
    \ds \kl{p(\vz|\vx)}{p(\vz)}  &= \ds \mathbb E_{p(\vz|\vx)} [ \log p(\vz|\vx) ] - \ds \mathbb E_{p(\vz|\vx)} [ \log p(\vz) ] \\
    &= - \ds \mathbb E_{p(\vz|\vx)} [ \log p(\vz) ] = - \int~d\vz \ds [p(\vz|\vx) \log p(\vz)] \\
    &= - \int~d\vz ( \ds p(\vz|\vx) [\log p(\vz)] + p(\vz|\vx) \ds \log p(\vz) )
\end{align*}
The first term in the first line is a constant that does not depend on $\snr$. Then we expand with the product rule. 
The following is an easily verified identity for our Gaussian noise channel, $p(\vz|\vx)$. 
$$\ds p(\vz|\vx) = -\vx / (2\sqrt{\snr}) \cdot \nabla_{\vz} p(\vz|\vx) $$
We also need the following relation. 
\begin{align*}
    \ds \log p(\vz) &=    1/p(\vz) \ds p(\vz) = 1/p(\vz) \int d \bar \vx \ds p(\vz| \bar \vx) p( \bar \vx) \\
    &= - 1/(2 \sqrt{\snr}) ~ 1/p(\vz) \int d \bar \vx ~  \bar \vx  \cdot \nabla_{\vz} p(\vz| \bar \vx) p( \bar \vx)
\end{align*}
Using these expressions in the derivation above, we get the following. 
\begin{align*}
    \ds \kl{p(\vz|\vx)}{p(\vz)}  &=  \int~d\vz \bigg( \vx / (2\sqrt{\snr}) \cdot \nabla_{\vz} p(\vz|\vx) [\log p(\vz)]  \\
      & \qquad + 1/(2 \sqrt{\snr}) ~ p(\vz|\vx)/p(\vz) \int d \bar \vx ~  \bar \vx  \cdot \nabla_{\vz} p(\vz| \bar \vx) p( \bar \vx) \bigg)
\end{align*}
Next we use integration by parts on $\vz$. 
\begin{align*}
    \ds \kl{p(\vz|\vx)}{p(\vz)}  &=  -\int~d\vz \bigg( 1/2 ~p(\vz|\vx)~ \vx / \sqrt{\snr} \cdot \nabla_{\vz}  \log p(\vz)  \\
      & \qquad + 1/(2 \sqrt{\snr}) ~ 1 / p(\vx) \int d \bar \vx ~ p(\vz| \bar \vx) p( \bar \vx)~ \bar \vx  \cdot \nabla_{\vz} p(\vx | \vz)  \bigg) \\
      & = -\int~d\vz \bigg( 1/2 ~p(\vz|\vx)~ \vx / \sqrt{\snr} \cdot \nabla_{\vz}  \log p(\vz)  \\
      & \qquad + 1/(2 \sqrt{\snr}) ~ p(\vz) / p(\vx) \mathbb E[\vx | \vz]  \cdot \nabla_{\vz} p(\vx | \vz)  \bigg)
\end{align*}
The gradients can be written, 
$$\nabla_{\vz} p(\vx|\vz) =  p(\vx|\vz) \sqrt{\snr} (\vx - \mathbb E[\vx | \vz])$$
$$\nabla_{\vz} \log p(\vz) = \sqrt{\snr} \mathbb E[\vx | \vz] - \vz.$$ 
\begin{align*}
    \ds \kl{p(\vz|\vx)}{p(\vz)}  & = \int~d\vz \bigg( 1/2 ~p(\vz|\vx)~ \vx / \sqrt{\snr} \cdot   (\vz - \sqrt{\snr} \mathbb E[\vx | \vz])  \\
      & \qquad - 1/(2 \sqrt{\snr}) ~ p(\vz| \vx)  \mathbb E[\vx | \vz]  \cdot  \sqrt{\snr} (\vx - \mathbb E[\vx | \vz]) \bigg) \\
      & = 1/2 \int~d\vz p(\vz|\vx) \bigg( - \vx  \cdot \mathbb E[\vx | \vz] + \vx \cdot \vz / \sqrt{\snr}  \\
      & \qquad -  \mathbb E[\vx | \vz]  \cdot   (\vx - \mathbb E[\vx | \vz]) \bigg) \\ 
      & = 1/2 \int~d\vz p(\vz|\vx) \bigg( - \vx  \cdot \mathbb E[\vx | \vz] + \vx \cdot \vx  \\
      & \qquad -  \mathbb E[\vx | \vz]  \cdot   (\vx - \mathbb E[\vx | \vz]) \bigg) \\
            & = 1/2 \int~d\vz \, p(\vz|\vx) \, \norm{ \vx - \mathbb E[\vx | \vz] }  \\
            &= \half \mmse(\vx, \snr) \qed
\end{align*}

\subsection{High SNR KL divergence limit}\label{app:proof2} 
In this appendix, we derive the following result. 
\begin{align*}
\lim_{\snr \rightarrow \infty} \kl{p(\vz|\vx)}{p_G(\vz)} - \kl{p(\vz|\vx)}{p(\vz)} = \log \frac{p(\vx)}{p_G(\vx)}
\end{align*}
In this expression, we have two ``base distributions'', $p(\vx), p_G(\vx)$, and we consider the marginal distributions after injecting Gaussian noise, $p_G(\vz) = \int p(\vz|\vx) p_G(\vx) d\vx $, $p(\vz) = \int p(\vz|\vx) p(\vx) d\vx $. 
Start by expanding and canceling out terms. 
\begin{align*}
 f(\vx, \snr) &\equiv \kl{p(\vz|\vx)}{p_G(\vz)} - \kl{p(\vz|\vx)}{p(\vz)} \tag{Cancel $\log p(\vz|\vx)$ terms} \\
     &= \mathbb E_{p(\vz|\vx)} [\log \big((p(\vz)  \snr^{d/2}) / (p_G(\vz) \snr^{d/2})\big)] \tag{Multiply by 1} \\
    &= \mathbb E_{p(\eps)} [\log \big((p(\vz=\rootsnr \vx + \eps)  \snr^{d/2}) / (p_G(\vz=\rootsnr \vx + \eps) \snr^{d/2})\big)]  \tag{Reparametrize}
\end{align*}
Set $Z = (2 \pi)^{d/2}$ for $\vx \in \mathbb R^d$ and re-arrange. 
\begin{align*}
    p(\vz) \snr^{d/2} &= \int d\bar\vx ~p(\bar\vx) 1/Z~e^{-\half (\vz - \rootsnr \bar \vx)^2 }  \snr^{d/2} \\
    &= \int d \bar\vx ~p(\bar \vx) {\color{blue} 1/Z~e^{-\half (\vz / \rootsnr - \bar\vx)^2 \snr}  \snr^{d/2} } \\
\end{align*}
    For large $\snr$, we recognize a limit representation of the Dirac delta function in blue. 
    \begin{align*} 
    p(\vz=\rootsnr \vx + \eps) \snr^{d/2}   &= \int d\bar \vx ~p(\bar \vx) {\color{blue} 1/Z~e^{-\half (\vx +\eps/\rootsnr - \bar\vx)^2 \snr}  \snr^{d/2} } \\
    \lim_{\snr \rightarrow \infty} p(\vz=\rootsnr \vx + \eps) \snr^{d/2} &= \int d\bar \vx ~p(\bar \vx) {\color{blue} \delta( \vx - \bar \vx )} \\
    &= p(\vx)
      \end{align*}

Using this result leads to the desired result.
\begin{align*}
\lim_{\snr \rightarrow \infty} f(\vx, \snr)    &= \lim_{\snr \rightarrow \infty} \mathbb E_{p(\eps)} [ \log p(\vz=\rootsnr \vx + \eps) \snr^{d/2} - \log p(\vz=\rootsnr \vx + \eps) \snr^{d/2}] \\
                                               &= \log p(\vx) / p_G (\vx)  \qed
\end{align*}
This informal proof using delta functions could be made more rigorous in a measure-theoretic setting. 
Once again, we have derived a point-wise generalization of Eq. 177 from \citet{guo}, which can be recovered by taking the expectation of our result.

\subsection{Simplifying Density with an Integral Identity}\label{app:identity}

Our goal is to simplify the expression of the density coming from Eq.~\eqref{eq:density}. 
\begin{align*}
-\log p(\vx) &= {\color{red} -\log p_G(\vx)} - \half \int_{0}^{\infty} d\snr \left( {\color{blue}\mmse_G(\vx, \snr)} - \mmse(\vx, \snr) \right) \\
&= {\color{red}d/2 \log(2 \pi) + x^2/2}   - \half \int_{0}^{\infty} d\snr \left( {\color{blue}\frac{x^2 + \snr~d}{(1+\snr)^2}} - \mmse(\vx, \snr) \right)  \\
 &= d/2 \log(2 \pi e) - \half \int_{0}^{\infty} d\snr \left(  \frac{d}{1+\snr} - \mmse(\vx, \snr) \right) 
\end{align*}
In the second line we write out the expressions. Note the the pointwise MMSE for a standard normal input is derived in more detail in App.~\ref{app:gaussian}. 
In the third line we make use of the following integral identity. 
\begin{equation*}
\int_0^\infty \left(\frac{x^2 + \snr ~d}{(1+\snr)^2} - \frac{d}{1+\snr} \right) d \snr = x^2 - d 
\end{equation*}
This identity can be verified via elementary manipulations (multiply second term in integrand by $(1+\snr)/(1+\snr)$), but we state it explicitly because of its counter-intuitive form. 

\subsection{Non-Gaussian Density Representation}\label{app:density_data}

Instead of using $p_G(\vx) = \mathcal N(\vx; 0, \mathbb I)$, we can take the base measure to be a Gaussian, $p_G(\vx) = \mathcal N(\vx; \bm \mu, \bm \Sigma)$, with mean and covariance that match the data. Let $\lambda_1,\ldots, \lambda_d$ be the eigenvalues of the covariance matrix. We could start with the derivation in App.~\ref{app:identity} and try to simplify in a similar way with a more complex integral identity. However, a more straightforward derivation proceeds as follows. 

Start with the simple density estimator from Eq.~\eqref{eq:density_simple} assuming a standard normal as the base measure, but replacing the factor of $d$ with a sum over $d$ terms.
$$-\log p(\vx) = d/2 \log(2 \pi e) - \half \int_{0}^{\infty} d\snr \left(  \sum_{i=1}^d \frac{1}{1+\snr} - \mmse(\vx, \snr) \right)$$
Now, note the following integral identity:
$$
\int_{0}^\infty \left( \frac{1}{\snr + 1/\lambda_i} - \frac{1}{\snr + 1} \right) d\snr = \log \lambda_i .
$$
Using this to replace $1/(\snr+1)$ terms in the previous expression gives the following. 
$$-\log p(\vx) = d/2 \log(2 \pi e) + \half \sum_{i=1}^d \log \lambda_i - \half \int_{0}^{\infty} d\snr \left(  \sum_{i=1}^d \frac{1}{\snr + 1/\lambda_i} - \mmse(\vx, \snr) \right)$$
We recognize the sum of the log eigenvalues as equivalent to the log determinant. The constants can also be pulled inside the determinant. 
\begin{equation*}
-\log p(\vx) = \underbrace{1/2 \log \det(2 \pi e \bm \Sigma)}_{\text{Gaussian entropy}} - \underbrace{\half \int_{0}^{\infty} d\snr \left(  \sum_{i=1}^d \frac{1}{\snr+1/\lambda_i} - \mmse(\vx, \snr) \right)}_{\text{Deviation from Gaussianity} \geq 0}
\end{equation*}
Note that the deviation from Gaussianity need only be non-negative in expectation.

When we change variables of integration to $\logsnr = \log \snr$ we get the following. 
\begin{equation}\label{eq:mmse_G_alpha}
\int_{0}^{\infty} d\snr  \sum_{i=1}^d \frac{1}{\snr+1/\lambda_i} = \int_{-\infty}^{\infty} d\logsnr  \sum_{i=1}^d \sigma(\logsnr + \log \lambda_i)
\end{equation}
Here the traditional sigmoid function is used $\sigma(t) = 1/(1+e^{-t})$.

\subsection{Gaussian properties}\label{app:gaussian} 

Consider $p(\vx) = \mathcal N(x; \bm \mu, \bm \Sigma)$, for some covariance matrix $\bm \Sigma$. Let $\bm \Sigma = \bm U \bm \Lambda \bm U^T$ be the SVD, with eigenvalues $\Lambda_{ii} = \lambda_i$. 
For this distribution we would like to derive the ground truth decoder and MMSE, both for testing purposes and to use as a fallback estimator in regions where our discovered decoder is clearly sub-optimal.

We have already mentioned in Eq.~\eqref{eq:opt} that the ideal decoder is:
$$\xhat^*(\vz, \snr) \equiv \arg \mmse(\snr)  = \mathbb E_{\vx \sim p(\vx |\vz) }[ \vx].$$
For $p(\vx)$ a Gaussian distribution, we can simply look up this conditional mean in a textbook to find that the MMSE estimator is:
\begin{align*}
\xhat^*_G(\vz, \snr) &= \bm \mu + \sqrt{\snr} (\snr \mathbb I + \bm \Sigma^{-1})^{-1} (\vz - \rootsnr \bm \mu) \\
& = \bm \mu + \sqrt{\snr} \bm U(\snr+\bm \Lambda^{-1})^{-1} \bm U^T (\vz - \rootsnr \bm \mu)
\end{align*}

Taking the expectation, $\mathbb E_{\vx, \vz} [(\vx-\xhat_G^*(\vz, \snr))^2]$, gives the MMSE with some manipulation:
$$
\mmse_G(\snr) = Tr\left( (\bm \Sigma^{-1} + \snr \mathbb I)^{-1} \right) = \sum_i \frac{1}{1/\lambda_i + \snr}.
$$
Note that we used the cyclic property of the trace to write the final expression in terms of the eigenvalues only. The case for the standard normal distribution~\citep{guo} follows from setting the eigenvalues to 1. 

The negative log likelihood is:
$$
-\log p(\vx) = \half ~ (\vx-\bm \mu)^T \bm \Sigma^{-1} (\vx-\bm \mu) + \half \log \det(2 \pi \bm \Sigma)
$$
It can be verified using the integral identities in App.~\ref{app:density_data} and App.~\ref{app:identity} that using the Gaussian MMSE in Eq.~\eqref{eq:density_simple} recovers this equation.

\subsection{MMSE Upper Bounds via Bregman Divergence Interpretation}\label{app:bregman}
In this section, we derive the upper bound in the objective using a suboptimal denoising function $\hat{\vx}(\vz, \snr)$. 
The Bregman divergence generated from the strictly convex function $\phi(\vx) = \frac{1}{2} \langle \vx, \vx \rangle  = \frac{1}{2} \| \vx \|^2_2$ is  
\small 
\begin{align}
    \hspace*{-.2cm} d_{\phi}(\vx, \vy) &= \frac{1}{2} \langle \vx, \vx \rangle - \frac{1}{2} \langle \vy, \vy \rangle -  \langle \vx - \vy, \nabla \phi(\vy) \rangle = \frac{1}{2} \langle \vx, \vx \rangle -\frac{1}{2} \langle \vy, \vy \rangle - \langle \vx - \vy, \vy \rangle = \frac{1}{2}\| \vx- \vy \|^2_2 .  \nonumber
\end{align}
\normalsize
Note that the Legendre dual of $\phi$, or $\psi(\vy) = \sup_{\vx} \langle \vx, \vy \rangle - \frac{1}{2}\phi(\vx) = \frac{1}{2}\| \vy \|_2^2$, matches $\phi(\vx)$.   Although, in general, we only know that $d_\phi(\vx,\vy) = d_{\psi}(\vy,\vx)$, in this case the Bregman divergence is symmetric with $d_{\phi}(\vx,\vy) = d_{\psi}(\vx,\vy)$.

For the Gaussian noise channel $\vz = \sqrt{\snr} \vx + \eps$ with source $\vx \sim p(\vx)$, we treat the Bregman divergence as a distortion loss for reconstructing $\vx$ from observed samples $\vz$ using the denoising function $\hat{\vx}({\vz,\snr})$.   In contrast to the developments in the main text, here we 
consider the $\mmse$ as a function of each noisy sample $\vz \sim p(\vz)$ instead of input sample $\vx \sim p(\vx)$.

Performing minimization in the second argument, \citet{banerjee2005clustering} show that the arithmetic mean over inputs minimizes the expected Bregman divergence (regardless of the convex generator $\phi$)
\begin{equation}
    \hat{\vx}^*(\vz, \snr) = \mathbb{E}_{p(\vx|\vz)}[\vx]= \arg \min \limits_{\hat{\vx}(\vz, \snr)} \mathbb{E}_{p(\vx|\vz)}\left[ d_{\phi}(\vx, \hat{\vx}(\vz, \snr)) \right] \label{eq:argmin}
\end{equation}
At this minimizing argument, the expected divergence is the MMSE and corresponds to the conditional variance (see \citep{banerjee2005clustering} Ex. 5)
\small
\begin{align}
 \frac{1}{2} \mmse(\vz, \snr) = \min \limits_{\hat{\vx}(\vz, \snr)} \mathbb{E}_{p(\vx|\vz)}\left[ d_{\phi}(\vx, \hat{\vx}(\vz, \snr)) \right] &= \frac{1}{2} \mathbb{E}_{p(\vx|\vz)}\left[ \| \vx - \mathbb{E}_{p(\vx|\vz)}[\vx] \|^2_2 \right] =  \frac{1}{2} \text{Var}_{p(\vx|\vz)}[\vx] \nonumber
\end{align}
\normalsize
To prove that the mean provides the minimizing argument in Eq. \eqref{eq:argmin}, \citep{banerjee2005clustering} Prop. 1 consider the gap in the expected divergence for a suboptimal representative $\hat{\vx}(\vz, \snr)$.   This can be shown to yield another Bregman divergence, which in our case becomes
\begin{align*}
    \mathbb{E}_{p(\vx|\vz)}\left[ d_{\phi}(\vx, \hat{\vx}(\vz, \snr)) \right] - \mathbb{E}_{p(\vx|\vz)}\left[ d_{\phi}(\vx, \hat{\vx}^*(\vz, \snr)) \right] &=  d_{\phi}(\hat{\vx}^*(\vz, \snr),  \hat{\vx}(\vz, \snr)) \\
    & = \frac{1}{2} \|\hat{\vx}^*(\vz, \snr) -  \hat{\vx}(\vz, \snr) \|^2_2 \geq 0 .
\end{align*}
This allows us to derive the gap in the MMSE upper bounds in Sec. \ref{sec:implementation}, which arise from using suboptimal neural network denoisers $\hat{\vx}(\vz, \snr)$ instead of the true conditional expectation.  
\small
\begin{align}
\mathbb{E}_{p(\vz,\vx)}\big[
\norm{\vx - \hat{\vx}(\vz,\snr)} 
\big] = \underbrace{ \mathbb{E}_{p(\vz,\vx)}\big[
\norm{\vx - \hat{\vx}^*(\vz,\snr)}
\big] }_{\mmse(\snr)} + 
\underbrace{
\mathbb{E}_{p(\vz,\vx)}\big[\norm{\hat{\vx}^*(\vz,\snr) - \hat{\vx}(\vz,\snr) }\big]
}_{\text{estimation gap}}
. \label{eq:estimation_gap}
\end{align}
\normalsize 

Finally, \citet{banerjee2005clustering} Sec. 4 proves a bijection between regular Bregman divergences and exponential familes, so that minimizing a Bregman divergence loss corresponds to maximum likelihood estimation within a corresponding exponential family.  See their Ex. 9 demonstrating the case of the mean-only, spherical Gaussian family.   Notably, in this case, the natural parameters $\theta$ and expectation parameters $\eta$ are equivalent.
For decoding in our Gaussian noise channel using the squared error as a (Bregman) loss, we obtain a probabilistic interpretation of the MMSE optimization in Eq. \eqref{eq:argmin}  (\citet{banerjee2005clustering})
\begin{align}
    \min \limits_{\eta = \hat{\vx}(\vz, \snr)} \mathbb{E}_{p(\vx|\vz)}\left[ d_{\phi}(\vx, \eta) \right] = \max \limits_{\theta = \hat{\vx}(\vz,\snr)} \mathbb{E}_{p(\vx|\vz)}[ \log \mathcal{N}(\vx; \hat{\vx}(\vz,\snr), \sigma^2\mathbb{I})] 
\end{align}
where the Normal family is the appropriate exponential family $p(\vx;\theta)$ and the optimum is $\hat{\vx}^*(\vz,\snr) = \eta^* = \theta^* = \mathbb{E}_{p(\vx|\vz)}[\vx]$.
Finally, we can view the equality in Eq. \eqref{eq:estimation_gap} as an expression of the Pythagorean relation or $m$-projection (\citet{amari2016information} Sec. 1.6, 2.8, \citet{nielsen2018information}) in information geometry.  
In particular, $\theta^* = \hat{\vx}^*(\vz,\snr)$ is the projection of $p(\vx|\vz)$ onto the submanifold of fixed-variance, diagonal Gaussian distributions, and for a suboptimal $\theta = \hat{\vx}(\vz,\snr)$ we have
\begin{align*}
    D_{KL}[p(\vx|\vz) \| \mathcal{N}(\vx; \hat{\vx}(\vz,\snr), \sigma^2 \mathbb{I})] =  &D_{KL}[p(\vx|\vz) \| \mathcal{N}(\vx; \hat{\vx}^*(\vz,\snr), \sigma^2 \mathbb{I})] \\
    &+ D_{KL}[\mathcal{N}(\vx; \hat{\vx}^*(\vz,\snr), \sigma^2 \mathbb{I}) \| \mathcal{N}(\vx; \hat{\vx}(\vz,\snr), \sigma^2 \mathbb{I})].
\end{align*}
\normalsize
In future work, it would be interesting to consider using the I-MMSE relations for general Bregman divergences as in \citet{wang2013generalized, wang2014bregman}.

\subsection{Comparison with Variational Bounds}\label{app:variational}

In this section, we aim to relate the $\mmse$ terms in our bound to terms in the standard variational lower bound for diffusion models.     

\newcommand{\condvarmarkov}{\frac{\gamma_{t-1}-\gamma_t}{\gamma_{t-1}}}
\newcommand{\condvardata}{\frac{\gamma_{t-1}-\gamma_t}{\gamma_{t-1}}}

\paragraph{Forward and Reverse Processes}  We first review the standard forward and reverse processes used to define denoising diffusion models  \citep{jaschaneq, ddpm, vdm}.   Using the notation of \citet{vdm} in terms of SNR values $\snr_t = {\alpha_t^2}/{\sigma_t^2}$, we set $\sigma_t^2 = 1$ for simplicity and without loss of generality, since the eventual objective will depend only on $\snr_t$.
\begin{equation}
\begin{aligned}
    q(\bz_{\gamma_{0:T}}|\vx) &= q(\bz_{\gamma_0} | \vx) \prod \limits_{t=1}^T q(\bz_{\gamma_t}|\bz_{\gamma_{t-1}}) \\
    &\text{where} \quad q(\bz_{\gamma_t}|\bz_{\gamma_{t-1}}) \coloneqq \mathcal{N}\left(\bz_{\gamma_t} ; \sqrt{\frac{\gamma_t}{\gamma_{t-1}}} \bz_{\gamma_{t-1}}, \condvarmarkov \mathbb{I} \right) 
\end{aligned}
\label{eq:markov_fwd}
 \end{equation}

The Markovian time-reversal $q(\bz_{\gamma_{t-1}}|\bz_{\gamma_{t}})$ of the forward process is only Gaussian in the limit of $T \rightarrow \infty$ \citep{anderson1982reverse, feller2015theory}, in which case we interpret both processes as stochastic differential equations \citep{diffusion_sde}.   However, the conditional $q(\bz_{\gamma_{t-1}}|\bz_{\gamma_{t}}, \vx)$ is always Gaussian. Using Bayes rule,
\begin{align}
q(\bz_{\gamma_{t-1}}|\bz_{\gamma_{t}}, \vx) = \frac{q(\bz_{\gamma_t}|\bz_{\gamma_{t-1}}) q(\bz_{\gamma_{t-1}}|x)}{q(\bz_{\gamma_{t}}|x)} = \mathcal{N} \left(\bz_{\gamma_{t-1}}  ;  \sqrt{\frac{\gamma_{t}}{\gamma_{t-1}}} \bz_{\gamma_{t}} + \frac{  \gamma_{t-1} - \gamma_t}{\sqrt{\gamma_{t-1}} }  \vx  \, , \, \condvardata \mathbb{I} \right) \label{eq:data_reverse}
\end{align}
since each the forward process is Markovian $q(\bz_{\gamma_t}|\bz_{\gamma_{t-1}}, \vx) = q(\bz_{\gamma_t}|\bz_{\gamma_{t-1}})$ and each $q(\bz_{\gamma_{t}}|\vx)$ is Gaussian by construction of the noise channel
\begin{align}
q(\bz_{\gamma_{t}}|\vx) = \sqrt{\gamma_t} \,\vx + \eps . \label{eq:noise_channel}
\end{align}   
See \citet{vdm} App. A2 for example derivations of the Gaussian parameters.

Finally, consider defining a generative model using a variational reverse process.   
\citet{jaschaneq} (Sec. 2.2) choose to parameterize each conditional $p(\bz_{\gamma_{t-1}}|\bz_{\gamma_{t}})$ as a Gaussian, which is inspired by the 
fact that $q(\bz_{\gamma_{t-1}}|\bz_{\gamma_{t}})$ is Gaussian in the limit as $T \rightarrow \infty$,
\begin{align}
   p(\bz_{\gamma_{0:T}},\vx)=  p(\bz_{\gamma_T})  \prod \limits_{t=1}^T& p(\bz_{\gamma_{t-1}}|\bz_{\gamma_{t}}) \cdot p(\vx|\bz_{\gamma_0})  \\
   {\text{where}}  \quad 
    p(\bz_{\gamma_{t-1}}|\bz_{\gamma_{t}}) &\coloneqq q(\bz_{\gamma_{t-1}}|\bz_{\gamma_{t}}, \hat{\vx}(\bz_{\gamma_t}, \gamma_t)) \nonumber \\
   &= \mathcal{N} \left(\bz_{\gamma_{t-1}}  ;  \sqrt{\frac{\gamma_{t}}{\gamma_{t-1}}} \bz_{\gamma_{t}} +
   \frac{  \gamma_{t-1} - \gamma_t}{\sqrt{\gamma_{t-1}} }
   \hat{\vx}(\bz_{\gamma_t}, \gamma_t) \, , \, \condvardata \mathbb{I} \right)  
    \nonumber .
    \end{align}
Note that we have written $p(\bz_{\gamma_{t-1}}|\bz_{\gamma_{t}})$ in terms of a denoising function $\hat{\vx}(\bz_{\gamma_t}, \gamma_t)$.    While other forms are possible (e.g. \citet{vdm} Eq. 28), this expression will be useful to draw connections with 
the optimal denoising function $\hat{\vx}^*(\vz,\snr)$ found in the $\mmse$ expression.   %

\paragraph{Discrete Variational Lower Bound} 
We can now express the discrete variational lower bound using extended state space importance sampling \citep{ais, jaschaneq} 
\begin{align}
   & \log p(\vx) \geq \mathbb{E}_{q(\bz_{\gamma_{0:T}}|\vx)}\left[ \log \frac{p(\bz_{\gamma_{0:T}},\vx)}{q(\bz_{\gamma_{0:T}}|\vx)}\right] = \mathbb{E}_{q(\bz_{\gamma_{0:T}}|\vx)}\left[ \log \frac{p(\bz_{\gamma_T})  \prod_{t=1}^T p(\bz_{\gamma_{t-1}}|\bz_{\gamma_{t}}) \cdot p(\vx|\bz_{\gamma_0})}{q(\bz_{\gamma_0} | \vx) \prod_{t=1}^T q(\bz_{\gamma_t}|\bz_{\gamma_{t-1}})}\right] \nonumber \\
    &\phantom{\log p(\vx) = \mathbb{E}_{q(\bz_{\gamma_{0:T}}|\vx)}\left[ \log \frac{p(\bz_{\gamma_{0:T}},\vx)}{q(\bz_{\gamma_{0:T}}|\vx)}\right]} 
    = \mathbb{E}_{q(\bz_{\gamma_{0:T}}|\vx)}\left[ \log \frac{p(\bz_{\gamma_T})  \prod_{t=1}^T p(\bz_{\gamma_{t-1}}|\bz_{\gamma_{t}}) \cdot p(\vx|\bz_{\gamma_0})}{q(\bz_{\gamma_T} | \vx) \prod_{t=1}^T q(\bz_{\gamma_{t-1}}|\bz_{\gamma_{t}},\vx)}\right] \nonumber \\[1.5ex]
    &= \mathbb{E}_{q(\bz_{\gamma_{0}}|\vx)}\left[ \log p(\vx|\bz_{\gamma_0}) \right] - D_{KL}[q(\bz_{\gamma_{T}}|\vx) \| p(\bz_{\gamma_T}) ] 
    + \sum \limits_{t=1}^T D_{KL}[q(\bz_{\gamma_{t-1}}|\bz_{\gamma_{t}},\vx) \| p(\bz_{\gamma_{t-1}}|\bz_{\gamma_{t}})] \nonumber
\end{align}
\normalsize
Compared to the notation in the main text, note that the forward process $q(\bz_{\gamma_{0}}|\vx)$ and $q(\bz_{\gamma_{T}}|\vx)$ represent the Gaussian noise channels instead of $p(\bz_\snr|\vx)$.   Instead of the true denoising posterior  $p(\vx|\bz_{\snr_0})$ is a variational distribution in the above expression.
Our goal is now to relate the KL divergence terms, particularly the true
time-reversal process $q(\bz_{\gamma_{t-1}}|\bz_{\gamma_{t}})$ or its Gaussian approximation $p(\bz_{\gamma_{t-1}}|\bz_{\gamma_{t}})$, 
to the MMSE.

\newcommand{\ttwo}{\gamma_{t}}%
\newcommand{\tone}{\gamma_{t-1}}%
\newcommand{\gttwo}{\gamma_t}%
\newcommand{\gtone}{\gamma_{t-1}}%

\newcommand{\epsscaling}{\sqrt{\frac{\gtone-\gttwo}{\gtone}}}
\newcommand{\epsvar}{{\frac{\gtone-\gttwo}{\gtone}}}
\newcommand{\invepsscaling}{\sqrt{\frac{\gtone}{\gtone-\gttwo}}}
\newcommand{\sigmatwo}{\sigma_{\delta}}
\newcommand{\epstwo}{\eps_{\delta}}
\newcommand{\alphadelta}{\alpha_{\delta}}

\subsubsection{Relation to Incremental Noise Channel Proof of I-MMSE Relation}
The proof  of the I-MMSE relation
in \citet{guo} relies on the construction of an \textit{incremental channel} which successively adds Gaussian noise to the data.   We recognize this channel as being identical to the conditional 
$q(\bz_{\gtone}|\bz_{\gttwo}, \vx)$ in the limit as $\delta = \gtone - \gttwo \rightarrow 0$, and restate the proof using the notation above to highlight connections with terms in the variational lower bound.

Recall that the following distributions are Gaussian: each noise channel $q(\bz_{\gtone}|\vx)$ and $q(\bz_{\gttwo}|\vx)$ (Eq.~\eqref{eq:noise_channel}),
the forward conditional $q(\bz_{\gttwo}|\bz_{\gtone})$ 
(Eq.~\eqref{eq:markov_fwd}),
and the reverse data-conditional $q(\bz_{\gtone}|\bz_{\gttwo},\vx)$
(Eq.~\eqref{eq:data_reverse}).  
Using subscripts to distinguish different standard Gaussian variables, e.g. $\eps_{\gttwo} \sim \mathcal{N}(0, \mathbb{I})$,  we have the following relationships
\begin{align}
    \bz_{\tone} &= \sqrt{\gtone} \, \vx + \eps_{\tone}   \label{eq:ztone}
    \hfill && \text{(using Eq.~\eqref{eq:noise_channel})}
    \\
    &= \sqrt{\frac{{\gttwo}}{\gtone}} \bz_{\ttwo} + \frac{\gtone-\gttwo}{\sqrt{\gtone}} \vx +  \epsscaling \eps 
   \hfill && \text{(using Eq.~\eqref{eq:data_reverse})}
   \label{eq:conditional_channel} \\
    \bz_{\ttwo} 
    &= \sqrt{\gttwo} \, \vx + \eps_{\gttwo} \nonumber 
    \hfill  && \text{(using Eq.~\eqref{eq:noise_channel})} \\
   &=   \frac{\sqrt{\gttwo}}{\sqrt{\gtone}} \, \bz_{\tone} +  \sqrt{\frac{\gtone - \gttwo}{\gtone}} \, \epstwo 
    \hfill  && \text{(using Eq.~\eqref{eq:markov_fwd})}
    \label{eq:zttwo} 
    .
\end{align}
Up to change in notation, Eq.~\eqref{eq:ztone} and \eqref{eq:zttwo} match the construction of the incremental noise channel in Eq.~(30)-(31) of \citet{guo}, while Eq.~\eqref{eq:conditional_channel} matches their Eq.~(37).

\paragraph{Proof of I-MMSE Relation using Incremental Noise Channel}
Our goal is to take the limit as $T\rightarrow \infty$ or $\delta= \tone -\ttwo \rightarrow 0$, in order 
to recover the MMSE relation $\frac{d}{d\gamma_t} I(\vx;\bz_{\ttwo}) = \nicefrac{1}{2} \mmse(\vx,\gttwo)$.   The MMSE relation for the derivative is equivalent to $I(\vx;\bz_{\ttwo}) - I(\vx;\bz_{\tone})  = \nicefrac{\delta}{2} \mmse(\vx,\gttwo) + o(\delta)$ for small $\delta$.  Thus, we focus on the difference
 $I(\vx;\bz_{\tone}) - I(\vx;\bz_{\ttwo})$.    
 
 Using the chain rule for mutual information and the Markov property $\bz_{\ttwo} \, \perp\,  \vx \, |\, \bz_{\tone} $ (since $\vx \rightarrow   \bz_{\tone} \rightarrow \bz_{\ttwo} $), we have
\begin{align}
    I(\vx;\{\bz_{\tone}, \bz_{\ttwo}\}) = I(\vx; \bz_{\tone}) +  &\cancel { I(\vx; \bz_{\ttwo} | \bz_{\tone} ) }=
    I(\vx; \bz_{\ttwo}) +  I(\vx; \bz_{\tone} | \bz_{\ttwo})    \nonumber \\[1.5ex]
    \implies I(\vx; \bz_{\tone}) - I(\vx; \bz_{\ttwo}) &= I(\vx; \bz_{\tone} | \bz_{\ttwo}) \label{eq:chain_rule}
\end{align}
which allows us to restrict attention to $I(\vx; \bz_{\tone} | \bz_{\ttwo})$. Using the forward or target noise process, note that this mutual information compares the ratio of the time-reversed data-conditional $q(\bz_{\tone}|\bz_{\ttwo},\vx)$ and the time-reversed Markov process $q(\bz_{\tone}|\bz_{\ttwo})$,
\begin{equation}
\begin{aligned}
     I(\vx; \bz_{\tone} | \bz_{\ttwo}) &= 
     \mathbb{E}_{q(\bz_{\tone},\vx|\bz_{\gamma_{t}})}\left[\log \frac{q(\bz_{\tone},\vx|\bz_{\gttwo})}{q(\bz_{\tone}|\bz_{\ttwo}) q(\vx|\bz_{\ttwo})} \right] \\
     &=
     \mathbb{E}_{q(\bz_{\tone},\vx|\bz_{\gamma_{t}})}\left[\log \frac{q(\bz_{\gtone}|\bz_{\ttwo},\vx)}{q(\bz_{\gtone}|\bz_{\gttwo})} \right] .
\end{aligned} \label{eq:conditional_mi}
\end{equation}
Recalling Eq.~\eqref{eq:conditional_channel},
\begin{align}
    \bz_{\tone} = \sqrt{\frac{{\gttwo}}{\gtone}} \bz_{\ttwo} + \frac{\gtone-\gttwo}{\sqrt{\gtone}} \vx + \epsscaling \eps, \label{eq:conditional_channel2}
\end{align}
we can see that $I(\vx; \bz_{\tone} | \bz_{\ttwo})$ is the mutual information over a Gaussian noise channel $q(\bz_{\gtone}|\bz_{\ttwo},\vx)$, where $\bz_{\ttwo}$ is given and the input is drawn from the conditional distribution {$q(\vx | \bz_{\ttwo})$}.   The marginal output distribution $q(\bz_{\gtone}|\bz_{\gttwo})$, which is analogous to $p(\bz_{\gtone})$ for the unconditional channel marginal in Eq.~\eqref{eq:pimmse}, is not Gaussian in general.

Noting that the mutual information $I(\vx; \bz_{\tone} | \bz_{\ttwo}) = I(\vx;  c \bz_{\tone} | \bz_{\ttwo})$ is invariant to rescaling of $\bz_{\tone}$ by $c$, we consider
\begin{align}
   \invepsscaling \bz_{\tone} &= \sqrt{\frac{{\gttwo}}{\gtone}} \invepsscaling \bz_{\gttwo} + \frac{\gtone-\gttwo}{\sqrt{\gtone}} \invepsscaling \vx + \eps \nonumber \\
   &= \sqrt{\frac{\gttwo}{\gtone - \gttwo}} \bz_{\gttwo} + \sqrt{\gtone - \gttwo} \, \vx  + \eps  \ \label{eq:scaled}
\end{align}
in the limit as the SNR, $\delta = \gtone-\gttwo$, approaches 0.

 \newcommand{\lemmap}{p}
\begin{lemma}[\citet{guo} Lemma 1]\label{lemma:guo_lemma} 
For a Gaussian noise channel 
\begin{align}
    \bz = \sqrt{\delta} \vx + \eps
\end{align}
where $\vx \sim \lemmap(\vx)$, $\mathbb{E}[\vx^2] < \infty$, and $\eps \sim \mathcal{N}(0, \mathbb{I})$, the input-output mutual information in the limit as $\delta \rightarrow 0$ is given by
\begin{align}
    I(\bz;\vx) = \frac{\delta}{2} \mathbb{E}_{\lemmap(\vx)}\left[\left(  \vx - \mathbb{E}_{\lemmap(\vx)}[\vx] \right)^2 \right] + o(\delta) = \frac{\delta}{2} \Var(\vx) + o(\delta) 
    \label{eq:lemma1}
\end{align}
In particular, the mutual information is independent of shape of the channel input distribution $\lemmap(\vx)$.
\end{lemma}
\begin{proof} The proof proceeds by constructing a upper bound on the channel mutual information $D_{KL}[\lemmap(\bz|\vx)\| g(\bz)] = I(\bz;\vx) + D_{KL}[\lemmap(\bz)\| g(\bz)]$ where $g(\bz) \coloneqq \mathcal{N}(\bz; \mathbb{E}_{\lemmap(\vx)}[\sqrt{\delta}\vx], (\delta \Var[\vx] + 1 ) \mathbb{I})$ is a Gaussian distribution with the same mean and variance as $\lemmap(\bz) = \int \lemmap(\vx) \lemmap(\bz|\vx) d\vx$.    As $\delta \rightarrow 0$, it can be shown that $D_{KL}[\lemmap(\bz)\| g(\bz)] = o(\delta)$.
For each $\vx$, the divergence between Gaussians $\mathbb{E}_{p(\vx)}\big[ D_{KL}[\lemmap(\bz|\vx)\| g(\bz)]\big]$ is tractable, and the terms involving the mean cancel in expectation.  The $\half \log \frac{\Var_{g(\bz)}[\bz]}{\Var_{\lemmap(\bz|\vx)}[\bz]}$ term in the KL divergence contributes the variance term 
in Eq.~\eqref{eq:lemma1}, where $\Var_{\lemmap(\bz|\vx)}[\bz]=1$ and 
$\Var_{g(\bz)}[\bz] = \delta \Var[\vx] + 1 $
is chosen to be the same as the variance of the output marginal $p(\bz)$.  See \citet{guo} App. II for detailed proof.  
\end{proof}

Applying Lemma \ref{lemma:guo_lemma} for the Gaussian channel in Eq.\eqref{eq:conditional_channel2}, where the input distribution is $q(\vx|\bz_{\gttwo})$ and the mutual information is $I(\vx;\bz_{\gtone}|\bz_{\gttwo}) = I(\vx;  c \bz_{\tone} | \bz_{\ttwo})$ (see Eq.\eqref{eq:scaled}), yields the desired \textsc{mmse} relation.  Summarizing the reasoning above, we have
\begin{align}
I(\vx;\bz_{\gttwo}) - I(\vx;\bz_{\gtone}) &\overset{\eqref{eq:chain_rule}}{=} I(\vx;\bz_{\gtone}|\bz_{\gttwo})  \overset{\eqref{eq:lemma1}}{=} \frac{\delta}{2} \Var_{q(\vx|\bz_{\gttwo})}[\vx] + o(\delta)
\\
\implies
\quad    \frac{d}{d\gamma} I(\vx;\bz_{\gttwo}) &= \frac{1}{2} \mathbb{E}_{q(\vx,\bz_{\gttwo})}\left[ \big( \vx - \mathbb{E}_{q(\vx|\bz_{\gttwo})}[\vx] \big)^2 \right] \coloneqq \frac{1}{2} \mmse(\gttwo),
\end{align}
which proves the MMSE relation in Eq. \eqref{eq:mmse}.

\paragraph{Relation to Variational Lower Bound}
Compare our expression in Eq.~\eqref{eq:var}
\begin{equation}
\mathcal{L}^{\text{diff}}_\infty(\vx) \coloneqq 
\underbrace{\mathbb E_{p(\vza | \vx)} [-\log p(\vx | \vza)]
}_{\text{Reconstruction loss}}   \underbrace{- \half \int_{\snr_0}^{\snr_1} \mmse(\vx, \snr) d\snr}_{\text{Diffusion loss}}. \nonumber
\end{equation}
\normalsize
to Eq.~(11)-(12) in \citet{vdm} 
\begin{align}
\mathcal{L}^{\text{diff}}_T(\vx) \coloneqq \underbrace{- \sum \limits_{t=1}^T 
\mathbb{E}_{q(\bz_{\gamma_t}|\vx)} D_{KL}\big[q(\bz_{\gamma_{t-1}}|\bz_{\gamma_t}, \vx) \| p(\bz_{\gamma_{t-1}} | \bz_{\gamma_t}) \big] }_{\text{Diffusion loss}}.
\end{align}
\normalsize
The conditional Gaussian parameterization of $p(\bz_{\gamma_{t-1}} | \bz_{\gamma_t})$ is often justified by the fact that $q(\bz_{\gamma_{t-1}}|\bz_{\gamma_t}, \vx)$ is Gaussian \citep{ddpm,vdm}.  

However, from the information theoretic perspective, our goal should be to estimate the conditional mutual information $I(\vx;\bz_{\snr_{t-1}}|\bz_{\snr_t})$ in Eq.~\eqref{eq:conditional_mi}.
Let $p_G^*(\bz_{\gamma_{t-1}} | \bz_{\gamma_t}) = \mathcal{N}(\bz_{\gamma_{t-1}}; \mu( \bz_{\gamma_{t-1}}|\bz_{\gamma_t} ), \sigma^2( \bz_{\gamma_{t-1}}|\bz_{\gamma_t} ) )$ be the maximum likelihood Gaussian approximation to the channel output marginal $q(\bz_{\gamma_{t-1}}|\bz_{\gamma_t}) = \int q(\bz_{\gamma_{t-1}}|\bz_{\gamma_t}, \vx) q(\vx|\bz_{\gamma_t}) d\vx$   (i.e. with the same mean and variance).
Rewriting the mutual information in terms of an upper bound using this Gaussian marginal, 
\begin{align}
    I(\vx;\bz_{\snr_{t-1}}|\bz_{\snr_t}) &= \mathbb{E}_{q(\vx|\bz_{\gamma_{t}})}\left[ D_{KL}[q(\bz_{\gamma_{t-1}}|\bz_{\gamma_t}, \vx) \| q(\bz_{\gamma_{t-1}}|\bz_{\gamma_t}) ] \right] \\[1.5ex]
    &= \medmath{ \mathbb{E}_{q(\vx|\bz_{\gamma_{t}})}\left[ D_{KL}[q(\bz_{\gamma_{t-1}}|\bz_{\gamma_t}, \vx) \| p_G^*(\bz_{\gamma_{t-1}} | \bz_{\gamma_t}) ] \right]  - D_{KL}[q(\bz_{\gamma_{t-1}}|\bz_{\gamma_t}) \| p_G^*(\bz_{\gamma_{t-1}} | \bz_{\gamma_t}) ] }. \nonumber
\end{align}
In the continuous time limit as $\delta \rightarrow 0$ or $T \rightarrow \infty$, we have that $D_{KL}[q(\bz_{\gamma_{t-1}}|\bz_{\gamma_t}) \| p_G^*(\bz_{\gamma_{t-1}} | \bz_{\gamma_t})) ]=o(\delta)$ (as in the proof of Lemma \ref{lemma:guo_lemma}, where the marginal divergence $D_{KL}[p(\bz)\|g(\bz)]=o(\delta)$).
Thus, we have
\begin{align}
    I(\vx; \bz_{\tone} | \bz_{\ttwo}) 
     &= \medmath{\mathbb{E}_{q(\vx|\bz_{\gamma_{t}})}\left[D_{KL} [ q(\bz_{\gtone}|\bz_{\ttwo},\vx)\|p_G^*(\bz_{\gtone}|\bz_{\gttwo})]\right] -  \underbrace{ D_{KL} [ q(\bz_{\gtone}|\bz_{\ttwo})\|p_G^*(\bz_{\gtone}|\bz_{\gttwo})] }_{\text{ \small $o(\delta) $  as  $\delta \rightarrow 0$}}} \nonumber \\
     &\overset{\delta \rightarrow 0}{\rightarrow}  \mathbb{E}_{q(\vx|\bz_{\gamma_{t}})}\left[D_{KL} [ q(\bz_{\gtone}|\bz_{\ttwo},\vx)\|p_G^*(\bz_{\gtone}|\bz_{\gttwo})]\right]
\end{align}

To summarize, in the continuous time limit, the proof of the I-MMSE relation shows that estimating the reverse process $q(\bz_{\gtone}|\bz_{\ttwo})$ only requires a Gaussian variational family.   The optimal Gaussian approximation $p_G^*(\bz_{\gtone}|\bz_{\gttwo})$ requires the variance 
 of $q(\bz_{\gtone}|\bz_{\ttwo})$, which depends on the variance of the input $q(\vx|\bz_{\gamma_t})$ to the noise channel $q(\bz_{\gtone}|\bz_{\ttwo}, \vx)$ (as in the proof of Lemma \ref{lemma:guo_lemma}).  Evaluating this variance involves (learning) the conditional expectation or optimal denoiser $\hat{\vx}^*(\bz_{\gamma_t},\gamma_t) = \mathbb{E}_{q(\vx|\bz_{\gamma_t})} [\vx]$ at each SNR, which matches Eq.~\eqref{eq:opt} and leads to the optimization in Sec. \ref{sec:implementation}.

\section{Implementation Details}

\subsection{Discrete probability estimator bound}\label{app:discrete}

We derive some results that are useful for estimating upper bounds on discrete negative log likelihood. 
\begin{align*}
\mathbb E[-\log P(\vx)] &= 1/2 \int_0^\infty \mmse(\snr) d\snr \\
&=  1/2 \int_{\snr_0}^{\snr_1} \mmse(\snr) d\snr + 1/2 \left(\int_{0}^{\snr_0} + \int_{\snr_1}^\infty \right) \mmse(\snr) d\snr  \\
&\leq 1/2 \int_{\snr_0}^{\snr_1} \mathbb E_{\vz,\vx}[\norm{\vx - \xhat(\vz, \snr)}] d\snr + c(\snr_0, \snr_1)
\end{align*}

First, consider the left tail integral, $I_L = 1/2 \int_{0}^{\snr_0} \mmse(\snr) d\snr$. When the SNR is low, the distribution is approximately Gaussian, and we can use the Gaussian MMSE to get an upper bound. The Gaussian MMSE should be for a Gaussian with either the same variance or same covariance matrix as the data, to ensure that it is an upper bound on the true MMSE. The eigenvalues of the covariance matrix are denoted with $\lambda_i$.
The MMSE results for Gaussians are discussed and derived in App.~\ref{app:gaussian}. 
\begin{align*}
    I_L &= 1/2 \int_{0}^{\snr_0} \mmse(\snr) d\snr \\
        &\leq 1/2 \int_{0}^{\snr_0} \mmse_G(\snr) d\snr \\
        &= 1/2 \int_{0}^{\snr_0} \sum_i 1/(\snr + 1/\lambda_i) d\snr \\
    &= 1/2 \sum_i \log(1+\snr_0 \lambda_i)
\end{align*}

Next we consider the right tail integral, $I_R = 1/2 \int_{\snr_1}^{\infty} \mmse(\snr) d\snr$. In this regime, the noise is extremely small. Because the data is discrete, we can get a good estimate by simply rounding to the nearest discrete value. Let the distance between discrete values be $\Delta$. 
\begin{align*}
    I_R &= 1/2 \int_{\snr_1}^{\infty} \mmse(\snr) d\snr \\
    &\leq 1/2 \int_{\snr_1}^{\infty} \mathbb E_{\vx, \epsilon}[ \norm{ \vx - \xhat(\rootsnr \vx + \eps, \snr)}]  d\snr \\
    &= 1/2 \sum_{i=1}^d \int_{\snr_1}^{\infty} \mathbb E_{\vx, \epsilon}[ \norm{\vx_i - \xhat_i(\rootsnr \vx_i + \eps_i, \snr)}]  d\snr
    \end{align*}
    We separate the analysis into a contribution from each vector component. Next, the possible errors per components will be multiples of $\Delta$, whose appearance depends only on the size of the noise.
    \begin{align*}
        I_R & \leq 1/2 \sum_{i=1}^d \int_{\snr_1}^{\infty} \sum_{j=1}^{j_{\text{max}}} (\Delta ~ j)^2 2 P((j-1/2) \Delta \leq \eps_i/\rootsnr \leq (j+1/2) \Delta) d\snr \\
        & \leq d/2 \int_{\snr_1}^{\infty} \sum_{j=1}^{j_{\text{max}}} (\Delta ~ j)^2 2 P((j-1/2) \Delta \leq \eps_i/\rootsnr ) d\snr \\
        & \leq d \int_{\snr_1}^{\infty} \sum_{j=1}^{j_{\text{max}}} (\Delta ~ j)^2  e^{- (j-1/2)^2 \Delta^2 \snr / 2} d\snr  \tag{Use Gaussian Chernoff bound}\\
        & = d  \int_{\snr_1}^{\infty} \sum_{j=1}^{j_{\text{max}}}  \frac{(\Delta ~ j)^2}{(j-1/2)^2 \Delta^2 /2} e^{- (j-1/2)^2 \Delta^2 \snr_1 / 2} \\
        &= 2 d   \sum_{j=1}^{j_{\text{max}}}  \frac{j^2}{(j-1/2)^2 } e^{- (j-1/2)^2 \Delta^2 \snr_1 / 2} \tag{Do integral} \\
        &\leq ~4~ d~   \sum_{j=1}^{j_{\text{max}}}  e^{- (j-1/2)^2 \Delta^2 \snr_1 / 2} \tag{Bound $j$ term}
\end{align*}

To summarize, the tail bound constants for our upper bound on discrete likelihood are as follows.
$$\mathbb E[-\log P(\vx)] \leq 1/2 \int_{\snr_0}^{\snr_1} \mathbb E_{\vz,\vx}[\norm{\vx - \xhat(\vz, \snr)}] d\snr + c(\snr_0, \snr_1)$$
$$c(\snr_0, \snr_1) \equiv 1/2 \sum_{i=1}^d \log(1+\snr_0 \lambda_i) + 4~ d~   \sum_{j=1}^{j_{\text{max}}}  e^{- (j-1/2)^2 \Delta^2 \snr_1 / 2}$$
For practical integration ranges, these terms are nearly zero.

\subsection{Relationship Between log(SNR) and Timesteps}
To use pre-trained models in the literature with our estimator, we need to translate ``$t$'', a parameter representing time in a Markov chain that progressively adds noise to data, to a SNR. 
Referring to Eq. (9) and Eq. (17) in \citet{nichol2021improved}, it is easy to build up the mapping from $\logsnr = \log$(SNR) to timestep $t$ in IDDPM. The following derivation shows the relationship:

\begin{align*}
     & \text{sigmoid}(\logsnr) = \cos\left(\frac{\frac{t}{T}+s}{1+s}\cdot \frac{\pi}{2}\right)^2 / \cos\left(\frac{s}{1+s}\cdot \frac{\pi}{2}\right)^2  \\
    \Longrightarrow & \cos\left(\frac{\frac{t}{T}+s}{1+s}\cdot \frac{\pi}{2}\right) = \cos\left(\frac{s}{1+s}\cdot \frac{\pi}{2}\right) \sqrt{\text{sigmoid}(\logsnr)} \\
    \Longrightarrow & t = T\left(\arccos{\left(\cos\left(\frac{s}{1+s}\cdot \frac{\pi}{2}\right) \sqrt{\text{sigmoid}(\logsnr)}\right)} \frac{2(1+s)}{\pi}-s\right)\\
\end{align*}

For DDPM, the mapping between $\logsnr$ and $t$ is a little bit tricky because it's discrete. Firstly, we construct a one-to-one mapping between the variance $\beta_t$ and $t$. In \citep{ddpm}, the $\beta_t$ is scheduled in a linear way. Denote $\eta_t = 1- \beta_t$ and then we have a mapping between $\bar{\eta_t} = \prod_{s=1}^t \eta_s$ and $t$. We match the closest value $\bar{\eta_t}$ with sigmoid$(\alpha)$, and then find the corresponding $t$ of $\alpha$. Specifically, the mapping between $\log$(SNR) and $t/T$ (scaled from 0 to 1) is shown in Fig. \ref{fig:avst}. Our schedule is generated from CIFAR-10 dataset, which gives more diffusion steps for high $\log$(SNR). Moreover, the schedule adapts along with the dataset by computing scale and mean from it.
\begin{figure}[h]
    \centering
    \includegraphics[width=0.5\textwidth]{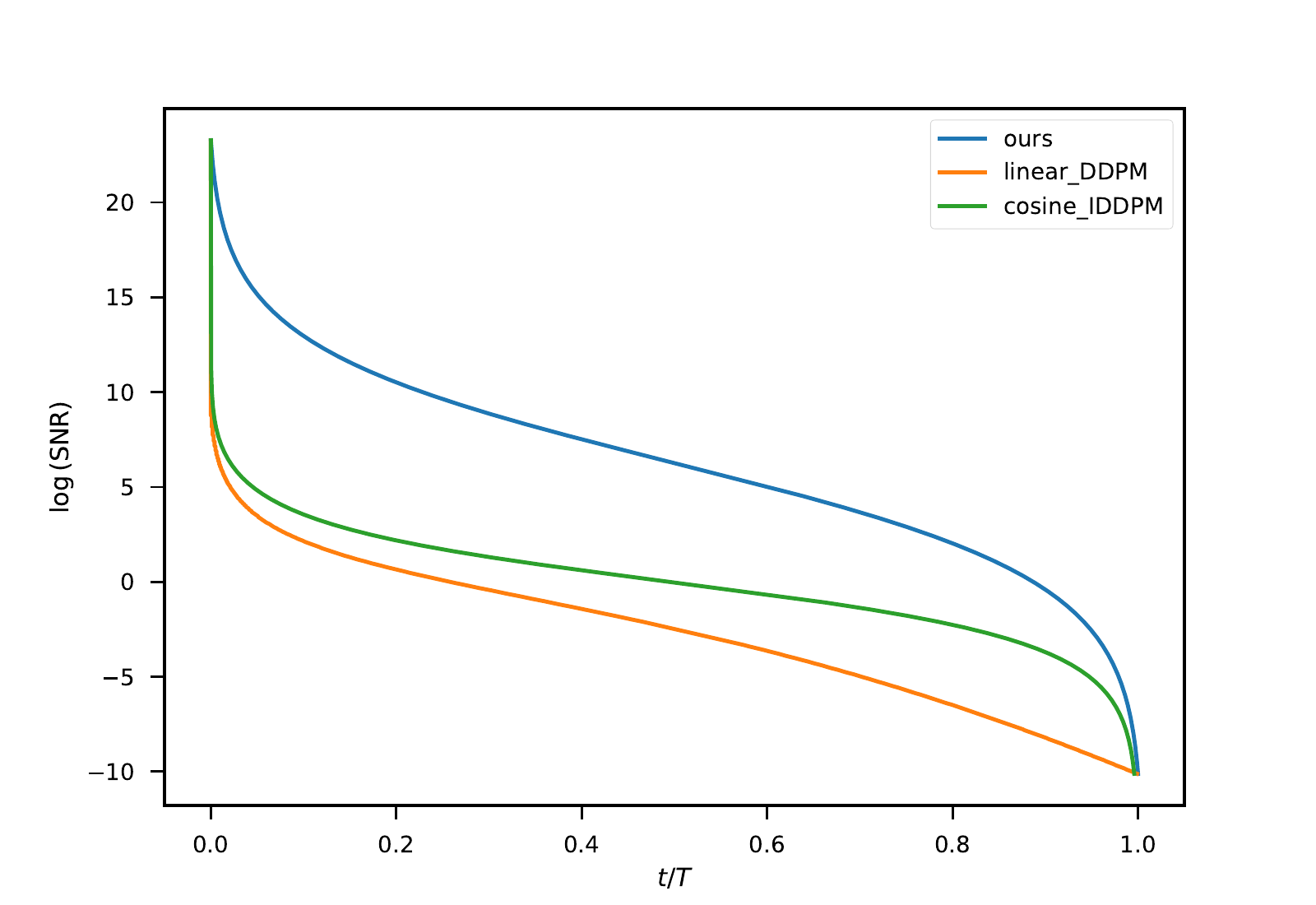}
    \caption{$\log$(SNR) v.s. timestep $t/T$}
    \label{fig:avst}
\end{figure}

\subsection{Details on Model Training}\label{app:tune}

In our fine-tuning experiments, we adopted a \href{https://github.com/openai/improved-diffusion}{IDDPM model} provided by \citep{nichol2021improved}, pre-trained on the unconditional CIFAR-10 dataset with the $L_{\text{vlb}}$ objective and cosine schedule.   We also consider a pre-trained {DDPM model} from Hugging Face (\url{https://huggingface.co/google/ddpm-cifar10-32}), which is provided by \citep{ddpm}.  We process CIFAR-10 $32 \times 32$ dataset in the same way as that in \url{https://github.com/openai/improved-diffusion/tree/main/datasets}. Since both models are trained with timesteps, we have to convert our $\log$(SNR) values to timesteps before passing them into models. The dataset is scaled to [-1, 1] for each pixel. 

For fine-tuning, we train each model for 10 epochs with the same `learning rate / batch size' ratio in the \citep{ddpm} and \citep{nichol2021improved}, e.g., `$ 10^{-4}/64$' for DDPM, and `$2.5 \times 10^{-5} / 32$' for IDDPM. During training, we reduce the learning rate by a factor of 10 and keep the same batch size after 5 epochs where the training loss starts to be flat. The optimizer for both models is Adam. In testing, we clip denoised images to range [-1, 1], but not during training. The fine-tuning results are displayed in Fig. \ref{fig:train_mse} and Fig. \ref{fig:train_loss}. It shows that the fine-tuning improves the NLL value by pushing the MSE curve down when the $\log$(SNR) is high. 
The NLL results are continuous NLL comparable to Table \ref{table:cont_nll}. 

\begin{figure}
	\begin{minipage}[t]{0.5\linewidth}
		\centering
		\includegraphics[width=1\textwidth]{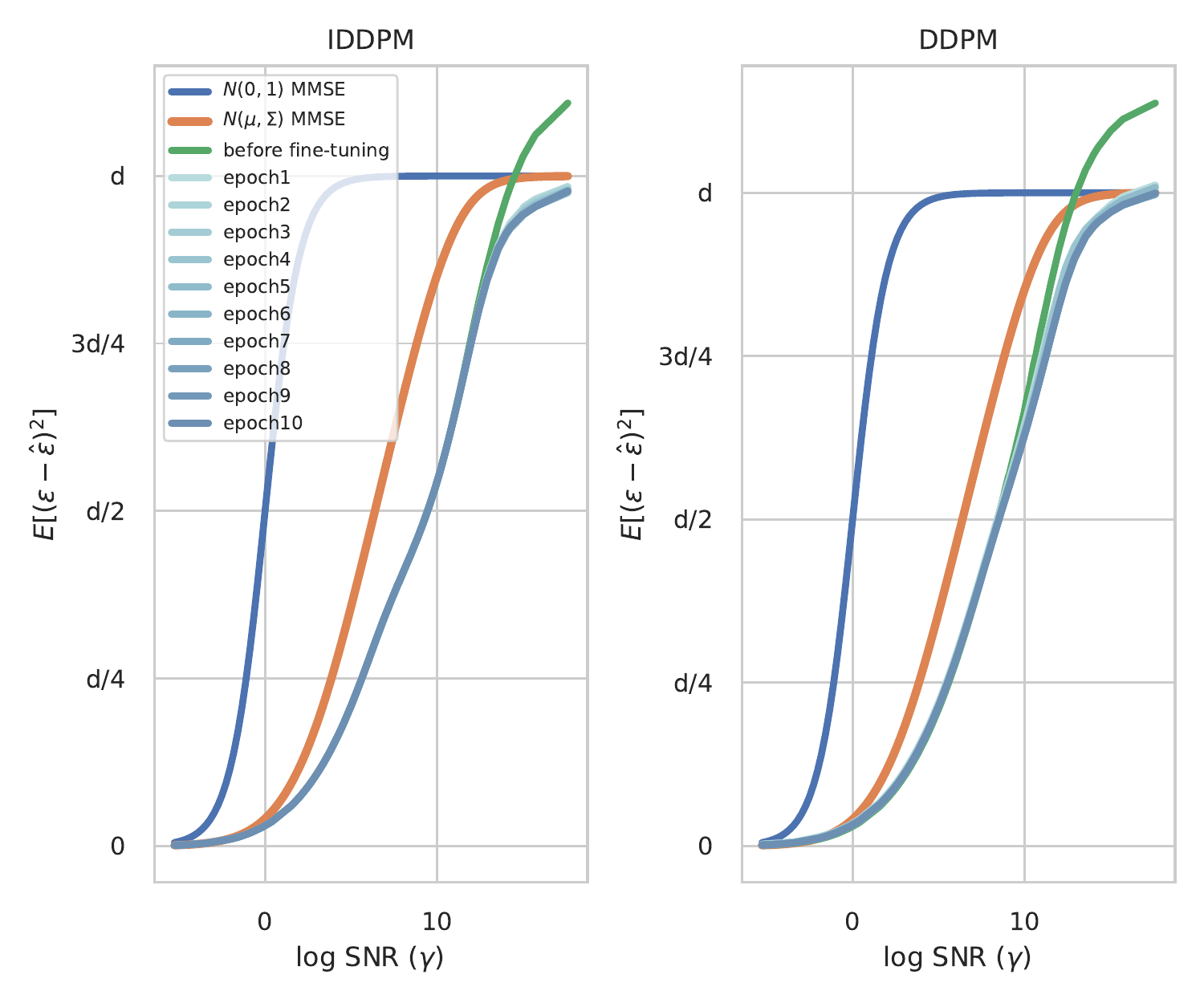}
        \caption{Change of denoising MSE curve during fine-tuning process.}
        \label{fig:train_mse}
	\end{minipage}
	\hspace{.02\linewidth}
	\begin{minipage}[t]{0.48\linewidth}
		\centering
		\includegraphics[width=1\textwidth]{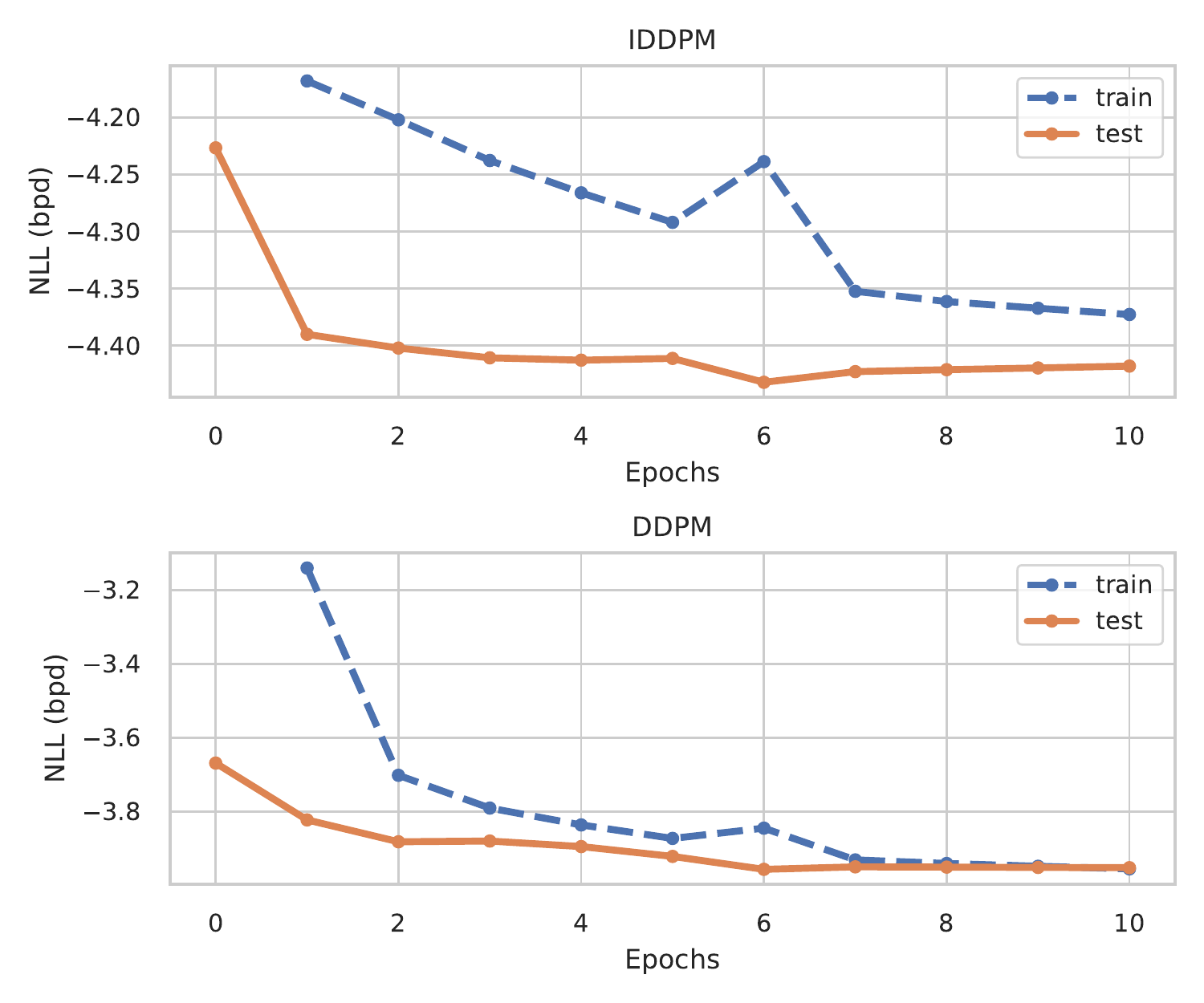}
        \caption{Training and testing NLL during fine-tuning process.}
        \label{fig:train_loss}
	\end{minipage}
\end{figure}

\subsection{Soft Discretization Function}\label{app:discretize}

We found in our experiments that existing denoising architectures were not able to learn the discreteness of the underlying data for the CIFAR dataset, which takes values $g_i = -1 + \Delta i$, with $\Delta = 2 / 255, i=0,\ldots, 255$. 
At high values of SNR, the noisy input, $z$, is very close to the true value of $x$, so zero prediction error can be achieved with high probability just by predicting the nearest discrete value. 
Therefore, we needed to add a discretization function to the denoiser value at high SNR to improve results. At first we tried a simple function that just rounded to the nearest discrete value. This function was overly aggressive in rounding near borderline values, which sometimes caused an increase in the mean square error. 
Therefore, we devised a ``soft discretization'' function. Besides improving the MSE, a soft discretization is also more amenable to being used as a differentiable nonlinearity in a trainable neural network, compared to the hard discretization function. 
\begin{figure}[htbp]
    \centering
    \includegraphics[width=0.7\columnwidth]{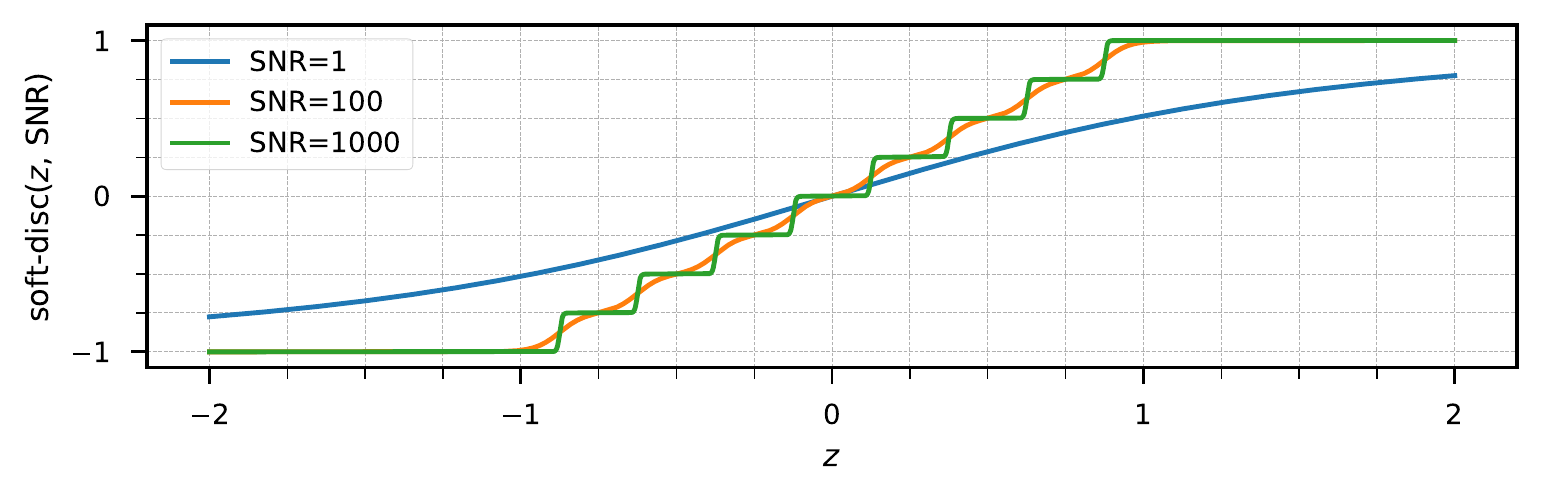}
    \caption{Illustrating the soft discretization function for various values of SNR ($\snr$), for regularly spaced discrete data, $x=-1, -0.75, -0.5, \ldots, 1$.}
    \label{fig:soft}
\end{figure}

Consider the noisy channel for scalar random variables, $z = \sqrt{\snr / (1+\snr)} x + \sqrt{1 / (1+\snr)} \epsilon$, where $\epsilon \sim \mathcal N(0,1)$ and $x \sim P(x)$. Let $P(x)$ be a distribution over $N$ discrete grid points, $g_1, \ldots, g_N$. Then we define the soft-discretization function as follows.
\begin{align*}
q(i) &= \mbox{softmax}_i( -0.5 (z - g_i)^2 (1+\snr)) \\
\mbox{soft-disc}(z; \snr) &= \sum_i g_i q(i) 
\end{align*}
The softmax can be interpreted as a distribution over nearby grid points, where the closest values are most likely. Then we simply take the expected grid point value as the output. The SNR plays the role of a temperature parameter that leads to a more strongly discretizing function at high SNR, and a more linear function at low SNR, as shown in Fig.~\ref{fig:soft}. 
The form of this nonlinearity was inspired by looking at the optimal case where $x$ is uniform, $P(x=g_i) = 1/N$. 

In Fig.~\ref{disc_density}, we see that this function is not optimal, leading to higher MSEs at low SNR values. However, it performs well at high SNR, so when we ensemble different denoisers we get the best results in Table~\ref{table:disc_nll}. 
A more effective strategy would be to replace $(1+\snr)$ in the soft discretization function with some learnable function, $f(\snr)$, then train the neural network accordingly. 
In this work, we wanted to see that we could improve density estimates over variational bounds by fine-tuning existing architectures with a new objective. Therefore, we used a fixed discretization function instead of a learnable one to avoid adding new parameters or complexity to the model. 

\subsection{Variance Estimates}\label{app:variance}

We now study two different types of variance estimates for our discrete and continuous log likelihood estimators. 
We use stochastic estimators for our upper bounds on the negative log likelihood in both cases. 
We would like to study the variance of these estimators to be sure that the estimates in our results do not look artificially low simply due to noise in the estimators.

\begin{table}[htbp]
    \centering
    \caption{The standard deviation of NLL (bpd) estimates with information-theoretic bounds}
    \begin{tabular}{*4c}
    \toprule
    & Continuous  & Discrete & Continuous + Dequantization \\    
    \midrule
    DDPM & 0.00035 & 0.00112 & 0.00033 \\
    IDDPM & 0.00051 & 0.00088 & 0.00046 \\
    \midrule
    DDPM(tune) & 0.00049 & 0.00100 & 0.00061 \\
    IDDPM(tune) & 0.00063 & 0.00087 & 0.00059 \\
    
    \bottomrule
    \end{tabular}
    \label{tab:var}
\end{table}

Our estimators for the discrete and continuous case are simple Monte Carlo estimates, rewritten in the following expressions with irrelevant constants discarded. 
\begin{align*}
 \mathcal L_{cont} &\propto ~ \mathbb E_{ \logsnr,\vx,\eps} \left[  \half ~1/q(\logsnr) (\norm{\eps - \epshat(\vz, \snr)}  - f_\Sigma(\logsnr) ) \right]  \\
 \mathcal L_{disc} &\propto ~ \mathbb E_{ \logsnr,\vx,\eps} \left[  \half~ 1/q(\logsnr) \norm{\eps - \epshat(\vz, \snr)}\right] 
\end{align*}
The random variables are $\eps \sim \mathcal N(0,\mathbb I), \vx \sim p(\vx), \logsnr \sim q(\logsnr)$, where $q(\logsnr)$ is the importance sampling distribution described in Sec.~\ref{sec:implementation}. Empirical estimates are taken by drawing samples from these distributions and computing an empirical mean. This produces an unbiased estimate that converges to the true expectation as we include more samples. Due to the central limit theorem, we know that the distribution of this estimator will converge to a Gaussian whose mean is the true expectation and whose variance is equal to the variance of samples divided by the total number of samples, $N$. We used this result to estimate the standard deviation of the estimators in Sec.~\ref{sec:experiments}, where we used $n=100$ samples of $\logsnr \sim q(\logsnr)$ each with $10000$ points from $\vx \sim p(\vx), \eps \sim N(0,1)$, so that $N=10^6$. For the discrete case, we used $n=1000$ samples from $q(\logsnr)$ per point in results in Sec.~\ref{sec:experiments}, but we report the result here using $n=100$ samples so that numbers are directly comparable.

The results are shown in Table~\ref{tab:var}, and we see that variance estimates are small compared to the values reported in Fig.~\ref{cont_density}. 
Note that the continuous density estimators have smaller variance than the discrete estimator. This makes sense because the importance sampling distribution that we used was chosen to match the integrand for the continuous density estimator. The more closely the importance sampling distribution matches the target, the lower the variance. 
The higher variance for the discrete estimator is due to the mismatch between the curves in Fig.~\ref{disc_density} and the logistic distribution, and we observe a similar phenomenon using the bootstrap estimators studied next. 
Therefore, we used more samples from $q(\logsnr)$ when using the discrete probability estimator. 

This analysis has an important caveat. These variance estimates reflect the case where all samples are drawn IID. However, for the purposes of plotting and ensembling, it is more convenient to have errors on a common set of $\log$(SNR) values, so that we can estimate average MSE's across samples per $\log$(SNR) ($\mathbb E_{p(\vx)}[\mmse(\vx, \snr)] = \mmse(\snr)$). Because we use the same set of $\log$(SNR) samples (drawn IID from $q(\logsnr)$) across \emph{all} samples, these samples are not fully IID. Hence, the standard error estimates in Table~\ref{tab:var}, which assume fully IID samples, are imperfect. This prompted us to include an alternate analysis based on bootstrap sampling.

\begin{table}[htbp]
    \centering
    \small
    \caption{Bootstrap variance estimates for continuous estimator}
    \begin{tabular}{*3c}
    \toprule
    & $n=1000$  & Bootstrap samples with $n=100$  \\    
    \midrule
    DDPM & -3.59 & -3.56 $\pm$ 0.04 \\
    IDDPM & -4.14 & -4.09 $\pm$ 0.10 \\
    \midrule
    DDPM(tune) & -3.87 & -3.84 $\pm$ 0.07 \\
    IDDPM(tune) & -4.31 & -4.28 $\pm$ 0.11 \\
    Ensemble & -4.31 & -4.29 $\pm$ 0.11 \\
    \bottomrule
    \end{tabular}
    \label{tab:boot1}
\end{table}

\begin{table}[h!]
    \centering
    \footnotesize
    	\begin{minipage}{0.45\linewidth}
		\centering
    \caption{Bootstrap variance estimates for discrete estimator}
    \begin{tabular}{*3c}
    \toprule
    & $n=1000$  & Boot. $n=100$  \\    
    \midrule
    DDPM & 3.68 & 3.56 $\pm$ 0.32 \\
    IDDPM & 3.16 & 3.05 $\pm$ 0.24 \\
    \midrule
    DDPM(tune) & 3.48 & 3.36 $\pm$ 0.28 \\
    IDDPM(tune) & 3.15 & 3.04 $\pm$ 0.24 \\
    Ensemble & 2.90 & 2.79 $\pm$ 0.27 \\
    \bottomrule
    \end{tabular}
    \label{tab:boot2}
    \end{minipage} \hfill
        	\begin{minipage}{0.45\linewidth}
		\centering
      \caption{Bootstrap variance estimates using uniform dequantization}
    \begin{tabular}{*3c}
    \toprule
    & $n=1000$  & Boot. $n=100$  \\    
    \midrule
    DDPM & 3.47 & 3.51 $\pm$ 0.04 \\
    IDDPM & 3.12 & 3.17 $\pm$ 0.07 \\
    \midrule
    DDPM(tune) & 3.37 & 3.41 $\pm$ 0.05 \\
    IDDPM(tune) & 3.13 & 3.18 $\pm$ 0.07 \\
    Ensemble & 3.11 & 3.16 $\pm$ 0.07 \\
    \bottomrule
    \end{tabular}
    \label{tab:boot3}
  \end{minipage}
\end{table}
\paragraph{Bootstrap sampling analysis}  

To get a sense for the variance of test time estimates with IID random log SNR values using our data consisting of a fixed set of random $\log$(SNR) values in common across samples, we did a bootstrap sampling analysis. For this analysis, we first constructed estimates using the full fixed set of $n=1000$ samples of $\logsnr \sim q(\logsnr)$ across all samples. Then, we took 10 random subsets of $n=100$ $\log$(SNR) values and again estimated NLL values. We show the mean and standard deviation across the bootstrap samples in each case.  The results for different estimators are shown in Tables \ref{tab:boot1}, \ref{tab:boot2}, \ref{tab:boot3}.

The use of non-IID $\log$(SNR)s does lead to higher variance. In principle, this could be avoided, even in the case when we are ensembling. To do so, we could use a validation set to determine which denoising model to use in each SNR range. Then, at test time, we could use fully IID samples for each $\log$(SNR) value. 

Note that the continuous estimators (including uniform dequantization, which uses a continuous estimator to give a discrete probability estimate) have far lower variance. This is why we used fewer $\log$(SNR) samples for continuous versus discrete estimators in the main results. 

Finally, note that the average bootstrap NLL using $n=100$ in the discrete case, for example 2.79 for the ensembling (Tab.~\ref{tab:boot2}), is actually lower than the best result reported in the main experiment section, NLL = $2.90$ using $n=1000$. We chose to report the larger but more reliable result using $n=1000$ since the variance of the discrete estimator using $n=100$ is large.

Diffusion models are computationally expensive and each point requires many calls per sample at test time to estimate log likelihood. Our variance analysis shows that our continuous density estimator (which can be used for discrete estimation also using uniform dequantization) can significantly reduce variance, leading to reliable estimates with far fewer model evaluations.

\end{document}